\documentclass[preprint,12pt,number,sort&compress]{elsarticle}

\usepackage{amssymb}
\usepackage{rotating}

\usepackage{amsmath,amsfonts}
\usepackage{amsthm}
\theoremstyle{definition}
\newtheorem{definition}{Definition}[section]

\theoremstyle{plain}
\newtheorem{assumption}{Assumption}[section]
\newtheorem{proposition}{Proposition}[section]

\theoremstyle{remark}
\newtheorem{remark}{Remark}[section]
\usepackage{subcaption}
\usepackage{xcolor}
\usepackage{hyperref}
\usepackage{graphicx}
\usepackage{booktabs}
\usepackage{longtable}
\usepackage{changepage}
\usepackage{multirow}
\usepackage{array}
\usepackage{tabularx}
\usepackage{comment}
\usepackage{algorithm}
\usepackage{algpseudocode}
\newcommand{\defeq}{\overset{\text{\tiny def}}{=}}
\begin{document}

\begin{frontmatter}

	\title{Post-Hoc Uncertainty-Aware Explanations for Deployed Power Quality Disturbance Classifiers via Laplace Approximation}

	\author{Yinsong Chen}
	\ead{yinsong.chen@deakin.edu.au}
	\affiliation{organization={School of Engineering, Deakin University},%Department and Organization
		city={Melbourne},
		postcode={3216},
		state={VIC},
		country={Australia}}
	
	\author{Samson S. Yu\corref{1}}
	\ead{samson.yu@deakin.edu.au}
	
	\affiliation{organization={School of Engineering, Deakin University},%Department and Organization
		city={Melbourne},
		postcode={3216},
		state={VIC},
		country={Australia}}
	
	\author{Kashem M. Muttaqi}
	\ead{kashem@uow.edu.au}
	\affiliation{organization={ARC Training Centre in Energy Technologies for Future Grids, School of Engineering, University of Wollongong},%Department and Organization
		city={Wollongong},
		postcode={2522},
		state={NSW},
		country={Australia}}
	
	\cortext[1]{Corresponding author}
	
\begin{abstract}
Deep learning classifiers achieve high accuracy in power quality disturbance (PQD) recognition, but deterministic explanation methods return a single attribution map and provide no measure of its uncertainty. This paper develops a post-hoc Bayesian explanation (B-explanation) method for trained PQD classifiers. A Laplace approximation constructs an approximate parameter posterior without retraining, and occlusion sensitivity propagates sampled functions into a distribution over attribution maps. We examine a diagonal approximation for the full network and three last-layer approximations with diagonal, Kronecker-factored, or full covariance. Percentile summaries and their band widths describe attribution variation across posterior samples. In the full-network results, low-percentile summaries increase relevance mass accuracy for swells and oscillatory transients, with lower intersection over union, while the remaining disturbance types show class-dependent behavior. Explanation dispersion increases under injected measurement noise, and qualitative examples reveal different component emphases within composite disturbances. Experiments on a synthetic benchmark of 15 disturbance classes and on field-recorded sags compare the method with Monte Carlo dropout and deep ensembles under a common evaluation protocol, evaluate it against deterministic occlusion, Local Interpretable Model-agnostic Explanations (LIME), and Shapley Additive Explanations (SHAP) with localization and faithfulness metrics, and characterize the computational cost of explanation generation for grid monitoring applications. Last-layer inference reduces online explanation-generation cost to a level comparable with deterministic occlusion in the measured setting.
\end{abstract}


\end{frontmatter}

\section{Introduction}
\label{sec:intro}
\begingroup

In modern power systems, power quality disturbances (PQDs), such as sags, swells, and harmonics, arise from power-electronic interfaces, distributed renewable generation, and nonlinear loads \cite{khan2023xpqrs,liu2025ensemble}. These disturbances can cause malfunction of sensitive equipment, protection misoperation, and reduced grid reliability \cite{quality1995ieee}. Automated PQD recognition is therefore used in grid-monitoring infrastructure, and its outputs increasingly feed protection, curtailment, and maintenance decisions in which classification errors are costly.

In conventional PQD recognition, features extracted from the waveform serve as inputs to a classifier. For example, Molu et al. used discrete wavelet decomposition to extract signal features and an extreme learning machine to identify the disturbance from those features \cite{molu2024enhancing}. Deep neural networks (DNNs) are now widely used for PQD recognition and can learn feature extraction and classification jointly. For instance, one-dimensional convolutional neural networks (CNNs) can learn discriminative features directly from raw voltage waveforms and exceed 99\% accuracy on synthetic benchmarks \cite{wang2019novel,machlev2021open}. Subsequent studies have improved recognition through dual-attention architectures \cite{liu2023dual}, multidimensional feature-driven ensembles \cite{liu2023multidimensional}, optimized empirical wavelet decompositions \cite{liu2023adaptive}, ensemble frameworks for distribution networks with photovoltaic generation and electric vehicles \cite{liu2025ensemble}, image-based transfer learning \cite{baig2024ensemble}, multimodal fusion of waveforms and scalograms \cite{baig2025multimodal}, and cross-domain generalization and denoising techniques \cite{xu2026power,wang2026power,kong2026power,agarwal2026cross}. The classification outputs do not directly identify the signal features supporting each decision, motivating explanation methods for monitoring applications.

Explainable artificial intelligence (XAI) addresses this limitation by attributing each decision to input features. In PQD analysis, attribution maps can be verified against physical signal characteristics, such as magnitude depressions for sags or high-frequency content for transients, and occlusion-sensitivity maps have been found to align well with disturbance intervals \cite{machlev2021measuring,machlev2022explainable}. However, deterministic XAI methods in power systems condition on a single parameter vector and return a single attribution map. Since DNNs exhibit a multimodal loss surface, many parameter configurations reach comparable accuracy while attending to different signal regions \cite{theus2025generalized}, so a single explanation does not show how attribution varies across these configurations.

Uncertainty-aware explanation has recently been studied outside power systems. Bykov et al.\ propagated Bayesian neural network (BNN) posteriors through layer-wise relevance propagation (LRP) \cite{bykov2020much,bykov2021explaining}. Wood et al.\ derived entropy-based importance measures for uncertainty-aware models \cite{wood2024model}. Mulye and Valdenegro-Toro reported variance of gradient explanations across ensemble members \cite{mulye2025uncertainty}. A method-agnostic general measure-theoretic framework was recently established \cite{chen2026unified}. This framework provides the theoretical foundation for explanation distributions. However, several questions remain for post-deployment PQD applications: deriving explanation distributions without retraining or ensembles, understanding how percentile summaries capture different disturbance characteristics, interpreting explanation uncertainty in physical terms, and quantifying the computational overhead of explanation generation.

This paper addresses these questions through a post-hoc Bayesian explanation (B-explanation) method for trained PQD classifiers. A Laplace approximation constructs a Gaussian parameter posterior around a trained network using curvature evaluated over the training data, and percentile summaries of the induced explanation distribution describe attribution variation across sampled functions. The main contributions are as follows:
\begin{itemize}
	\item A post-hoc B-explanation method that pushes a Laplace parameter posterior through the occlusion-sensitivity operator, yielding a distribution over disturbance-localization maps for a pre-trained network without retraining. We compare a diagonal approximation for the full network with three last-layer approximations that use diagonal, Kronecker-factored, or full covariance.
	\item An exploratory per-class percentile analysis over the 15 disturbance classes, distinguishing changes in attribution mass within disturbance intervals from changes in temporal overlap and relating these patterns to disturbance characteristics.
	\item Evidence that explanation dispersion increases under injected noise, together with qualitative examples of different component emphasis in composite disturbances and comparisons of predictive and explanation uncertainty under distribution shift.
	\item A computational analysis covering analytic and measured explanation-generation cost, comparisons with Monte Carlo dropout, deep ensembles, Local Interpretable Model-agnostic Explanations (LIME), and Shapley Additive Explanations (SHAP) under a common evaluation protocol with stated sampling budgets, and discussion of integration with supervisory control and data acquisition (SCADA) and edge-computing architectures.
\end{itemize}

Table~\ref{tb:related} summarizes related studies and the aspects addressed in this work.

\begin{table*}[!t]
	\begin{adjustwidth}{-2.5cm}{-2.5cm}
	\scriptsize
	\centering
	\setlength{\tabcolsep}{4pt}
	\renewcommand{\arraystretch}{1.15}
	\caption{Closely related studies, their contributions, and the limitations addressed in this work.}
	\label{tb:related}
	\begin{tabularx}{\linewidth}{p{3cm}p{7cm}p{7cm}}
		\toprule
		Study & Contribution & Limitation addressed in this work \\
		\midrule
		Wang and Chen \cite{wang2019novel} & Deep CNN on raw waveforms with $>$99\% accuracy on a 15-class benchmark & No interpretability or uncertainty support \\
		Machlev et al. \cite{machlev2021measuring,machlev2022explainable} & XAI evaluation of PQD classifiers, with occlusion maps showing the closest physical alignment & Deterministic explanations without an uncertainty measure \\
		Liu et al. \cite{liu2023dual,liu2023multidimensional,liu2023adaptive,liu2025ensemble} & Attention, ensemble, and empirical-wavelet-transform-based recognition of complex and multiple PQDs & Classification accuracy without explanation or uncertainty quantification \\
		Baig et al. \cite{baig2024ensemble,baig2025multimodal} & Image-transfer and multimodal deep ensembles for PQD classification & Ensembles evaluated for accuracy without explanation uncertainty \\
		Bykov et al. \cite{bykov2020much,bykov2021explaining} & Bayesian LRP with posterior percentiles of heatmaps for BNNs & LRP and Bayesian training in the image domain, without ground-truth localization \\
		Mulye and Valdenegro-Toro \cite{mulye2025uncertainty} & Variance of gradient explanations across uncertainty methods & No domain ground truth \\
		Wood et al. \cite{wood2024model} & Entropy-based importance for uncertainty-aware models & Global importance rather than instance-level localization \\
		Chen et al. \cite{chen2026unified} & General measure-theoretic framework for explanation distributions & General theory applied here to post-hoc Laplace inference, PQD-specific behavior and computational cost \\
		\bottomrule
	\end{tabularx}
	\end{adjustwidth}
\end{table*}

The remainder of the paper is organized as follows. Section~\ref{sec:1} presents the B-explanation methodology. Section~\ref{sec:2} defines the evaluation framework, including localization metrics, faithfulness metrics, and the statistical protocol. Section~\ref{sec:3} reports experimental results on synthetic and real-world data. Section~\ref{sec:discussion} discusses the physical interpretation of explanation uncertainty and deployment considerations. Section~\ref{sec:4} concludes the paper.
\par\endgroup

\section{Preliminaries and Methodology}
\label{sec:1}
\begin{figure*}[hbt]
	\centering
	\includegraphics[scale=1,width=\textwidth]{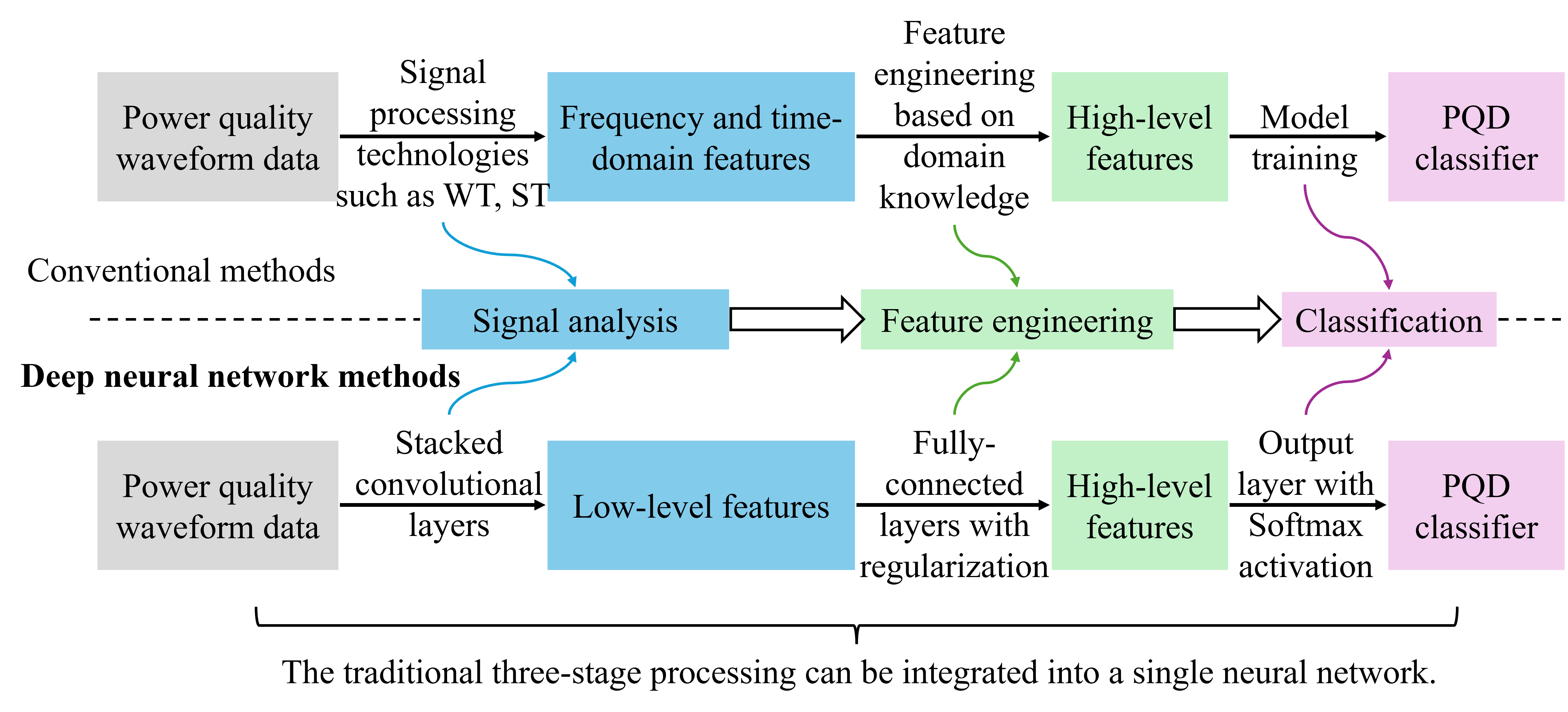}
	\caption{Three-stage framework for PQD classification using conventional and deep learning methods.}
	\label{fig:PQDclassification}
\end{figure*}

\subsection{Power Quality Disturbance Classification}
Let $x[n] \in \mathbb{R}$, $n=1,\dots,N$, denote a sampled voltage waveform exhibiting either a simple or a complex power quality disturbance, and let $x_0[n]$ denote the undisturbed reference,
\begin{equation}
	\label{eq:1}
	x_0[n] = A\sin(\omega n+\phi),
\end{equation}
where $A$, $\omega$, and $\phi$ are the amplitude, angular frequency, and phase. A \textbf{simple disturbance} is a deviation from $x_0$ caused by a single disturbance component (e.g., sag, transient), which may be localized or extend throughout the observation interval. A \textbf{complex disturbance} is the superposition or concatenation of $m>1$ simple disturbances. This study considers the 15 representative disturbance types of \cite[Table 1]{wang2019novel}, covering both cases. The classifier is a parametric map
\begin{equation}
	f_{\theta}: \mathcal{X} \rightarrow \Delta^{K-1},
\end{equation}
where $\mathcal{X}\subseteq\mathbb{R}^N$ is the input space, $\Delta^{K-1}$ is the probability simplex over the $K=16$ classes (the 15 disturbance types plus the undisturbed class), $\theta\in\Theta$ collects the network parameters, and the predicted label is $\hat{y}=\arg\max_k [f_{\theta}(x)]_k$.

\begin{table}
	\scriptsize
	\centering
	\caption{Architecture of the deep convolutional neural network (DCNN) of \cite{wang2019novel} used as the deployed classifier. ReLU denotes the rectified linear unit.}
	\label{tb:DCNN}
	\begin{tabular}{l|l|l}
		\hline
		Layer & Details & Activation\\
		\hline
		Convolution 1 & Kernel: 1 $\times$ 3, channels: 32, stride: 1 & ReLU \\
		\hline
		Convolution 2 & Kernel: 1 $\times$ 3, channels: 32, stride: 1 & ReLU \\
		\hline
		Max pooling 1 & Kernel: 1 $\times$ 3, stride: 1 & \\
		\hline
		Batch normalization 1 & Features: 32 & \\
		\hline
		Convolution 3 & Kernel: 1 $\times$ 3, channels: 64, stride: 1 & ReLU \\
		\hline
		Convolution 4 & Kernel: 1 $\times$ 3, channels: 64, stride: 1 & ReLU \\
		\hline
		Max pooling 2 & Kernel: 1 $\times$ 3, stride: 1 & \\
		\hline
		Batch normalization 2 & Features: 64 & \\
		\hline
		Convolution 5 & Kernel: 1 $\times$ 3, channels: 128, stride: 1 &  ReLU \\
		\hline
		Convolution 6 & Kernel: 1 $\times$ 3, channels: 128, stride: 1 &  ReLU \\
		\hline
		Max pooling 3 & Kernel: 1 $\times$ 624 & \\
		\hline
		Batch normalization 3 & Features: 128 & \\
		\hline
		Flatten 1 & & \\
		\hline
		Fully connected 1 & Features: 256 &  ReLU\\
		\hline
		Fully connected 2 & Features: 128 &  ReLU \\
		\hline
		Batch normalization 4 & Features: 128 & \\
		\hline
		Fully connected 3 & Features: 16 &  Softmax \\
		\hline
	\end{tabular}
\end{table}

A typical PQD classification process comprises signal analysis, feature engineering, and classification. Deep learning integrates the three stages within a single network trained end-to-end (Fig.~\ref{fig:PQDclassification}). Table~\ref{tb:DCNN} shows the DCNN of \cite{wang2019novel} adopted throughout this paper as the deployed classifier, i.e., the model assumed to exist before the explanation method is applied.

\subsection{Occlusion Sensitivity and Its Resolution}
\label{sec: occ}
XAI methods are divided into \textit{global} methods, which characterize model behavior over datasets, and \textit{local} methods, which attribute individual predictions to input features \cite{bykov2021explaining}. Safety-critical PQD monitoring requires instance-level auditing, so this work concerns local attribution.

\begin{definition}[Relevance Attribution Operator]
	\label{def:1}
	For a fixed input $x$, a relevance attribution operator maps the classifier under parameter realization $\theta$ to a relevance vector,
	\begin{equation}
		h_x(\theta) \defeq \mathcal{T}_{x, \theta}[f_{\theta}](x) = R_{\theta}(x) \in \mathbb{R}^N,
	\end{equation}
	where $R_{\theta}(x)[n]$ scores the contribution of $x[n]$ to the prediction of $f_{\theta}$.
\end{definition}

Among local methods, occlusion sensitivity \cite{zeiler2014visualizing} is adopted for three reasons. First, it is structurally matched to PQD signals: it slides a window across the waveform and measures the classifier response to the removal of contiguous segments, which mirrors disturbance localization. Second, prior evaluation on PQD classifiers found occlusion maps the most physically consistent among standard attribution methods \cite{machlev2021measuring}. Third, its relevance scores are bounded (Remark~\ref{rem:integrability}), which ensures finite moments for the mean, variance and empirical-distribution estimators. Percentile consistency additionally requires the conditions stated below.

Let $w$ be the occlusion-window length, $v$ the stride, and $b$ the baseline value written into the occluded segment. For window position $t$, the occluded signal $\tilde{x}^{(t)}$ replaces $x[vt+1], \ldots, x[vt+w]$ by $b$, and the window relevance for the class of interest $c$ is
\begin{equation}
	\label{eq:occ_window}
	R_t = [f_{\theta}(x)]_c - [f_{\theta}(\tilde{x}^{(t)})]_c .
\end{equation}
The per-sample relevance averages all windows covering index $n$:
\begin{equation}
	\label{eq:occ_avg}
	\begin{split}
		& h_x(\theta)[n] = \frac{1}{|T_n|} \sum_{t \in T_n} R_t,\\
		& T_n = \{t \mid n \in \{vt+1,\dots,vt+w\}\}.
	\end{split}
\end{equation}

At 3.2~kHz and a nominal frequency of 50~Hz, the 64-sample occlusion window spans one cycle. Occluded samples are replaced by the baseline $b=0$. The stride $v=8$ ($1/8$ cycle) samples occlusion positions every 2.5~ms while keeping their number moderate (Section~\ref{sec:complexity}). The window and stride are held fixed in the experiments. The target class is $c=y$ for labeled-disturbance localization or $c=\hat y$ for explaining the unperturbed deployed classifier's decision, and is held fixed across samples and occlusions.

\begin{figure*}[htbp]
	\centering
	\begin{subfigure}[b]{0.48\textwidth}
		\centering
		\resizebox{\linewidth}{2.7cm}{%
			\includegraphics[scale=1,width=\textwidth]{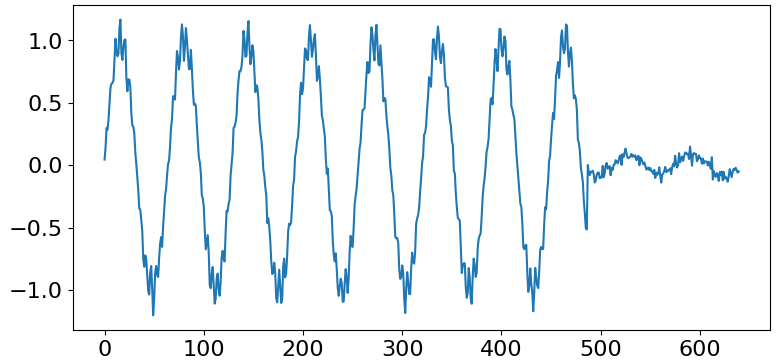}}
		\caption{}
	\end{subfigure}
	\begin{subfigure}[b]{0.48\textwidth}
		\centering
		\resizebox{\linewidth}{2.7cm}{%
			\includegraphics[scale=1,width=\textwidth]{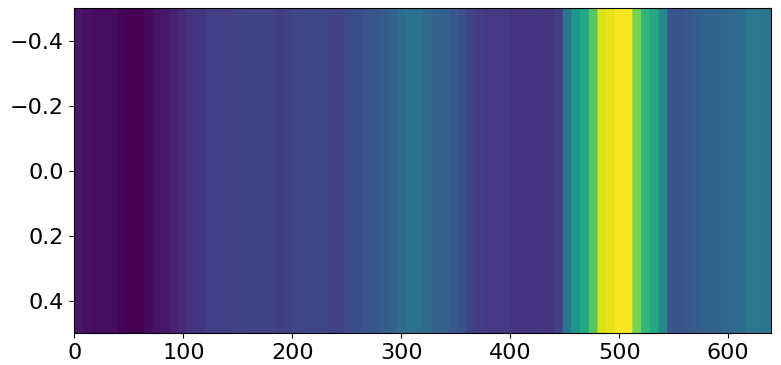}}
		\caption{}
	\end{subfigure}

	\begin{subfigure}[b]{0.48\textwidth}
		\centering
		\resizebox{\linewidth}{2.7cm}{%
			\includegraphics[scale=1,width=\textwidth]{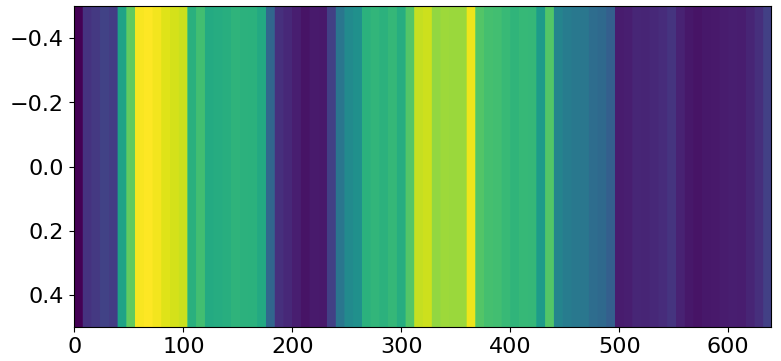}}
		\caption{}
	\end{subfigure}
	\begin{subfigure}[b]{0.48\textwidth}
		\centering
		\resizebox{\linewidth}{2.7cm}{%
			\includegraphics[scale=1,width=\textwidth]{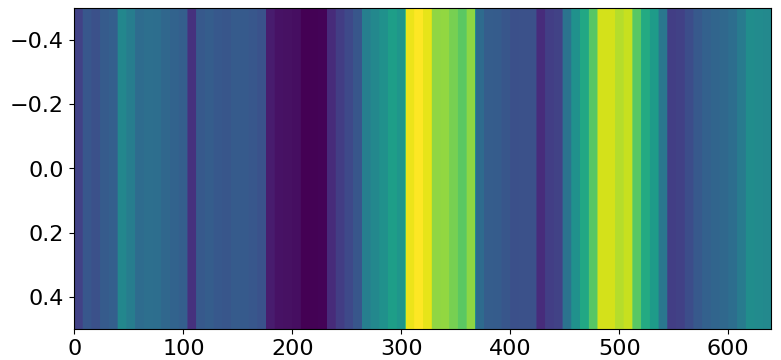}}
		\caption{}
	\end{subfigure}
	\caption{DNNs with identical architectures and similar performance can yield different explanations. Three models share the architecture in Table~\ref{tb:DCNN} but differ in parameters sampled from the Laplace posterior. All classify signal (a) correctly as ``Flicker+Sag'' with similar predictive entropy, but their explanations (b--d) differ.}
	\label{fig:example}
\end{figure*}

\subsection{The Explanation Distribution}
\label{sec:expl_dist}
The relevance vector $h_x(\theta)$ depends on the parameters $\theta$. Conventional XAI conditions on a single realization, typically the training optimum. However, DNNs are underdetermined: many parameter configurations achieve equivalent predictive performance while producing different attributions \cite{theus2025generalized}. Fig.~\ref{fig:example} shows three posterior samples of the DCNN that all classify a ``Flicker+Sag'' signal correctly yet disagree on whether the sag, the flicker, or both drive the decision. A single deterministic explanation corresponds to only one of these attribution patterns.

The Bayesian view represents parameter uncertainty by treating $\theta$ as a random variable. A categorical likelihood and an isotropic Gaussian prior with precision $\lambda$ are defined by
\begin{equation}
	\label{eq:prior}
	p(\mathcal{D}\mid\theta) = \prod_{i=1}^{M} [f_{\theta}(x_i)]_{y_i}, \qquad
	p(\theta) = \mathcal{N}(0, \lambda^{-1}I),
\end{equation}
for a labeled dataset $\mathcal{D} = \{(x_i, y_i)\}_{i=1}^{M}$ of size $M$. Bayes' theorem yields the posterior
\begin{equation}
	\label{eq:posterior}
	p(\theta\mid\mathcal{D}) = \frac{p(\mathcal{D}\mid\theta)\,p(\theta)}{p(\mathcal{D})},
	\qquad
	p(\mathcal{D}) = \int_{\Theta} p(\mathcal{D}\mid\theta)\,p(\theta)\,d\theta .
\end{equation}
Exact inference is intractable for DNNs, so an approximate posterior $q(\theta) \approx p(\theta\mid\mathcal{D})$ is used (Section~\ref{sec:laplace}). Throughout, $p(\theta\mid\mathcal{D})$ denotes the exact posterior, $q$ the approximate posterior, and push-forward measures are written with the operator $\#$.

\begin{assumption}[Measurable attribution operator]
	\label{ass:measurable}
	For each fixed $x$, the map $h_x:\Theta\to\mathbb{R}^N$ is Borel measurable.
\end{assumption}
For a finite ReLU network and fixed $x$, the parameter space partitions into finitely many Borel activation-pattern regions on which the occlusion map \eqref{eq:occ_window}--\eqref{eq:occ_avg} is continuous, so Assumption~\ref{ass:measurable} holds for the operator used here \cite{chen2026unified}.

\begin{definition}[Explanation Distribution]
	\label{def:expl_dist}
	The \emph{approximate explanation distribution} at input $x$ is the push-forward of $q$ through $h_x$,
	\begin{equation}
		\label{eq:expl_measure}
		Q_x \defeq (h_x)_{\#}q, \qquad
		Q_x(A) = q\!\left(h_x^{-1}(A)\right), \quad A \in \mathcal{B}(\mathbb{R}^N),
	\end{equation}
	where $\mathcal{B}(\mathbb{R}^N)$ is the Borel $\sigma$-algebra. The \emph{ideal explanation distribution} is $P_x \defeq (h_x)_{\#}\,p(\theta\mid\mathcal{D})$.
\end{definition}

$Q_x$ is a probability measure rather than a density: because $h_x$ maps a high-dimensional $\Theta$ into $\mathbb{R}^N$ through a nonlinear network, the push-forward need not admit a Lebesgue density, and the distribution is handled through its action on measurable sets and functions. The distribution is accessed computationally through sampling.

\begin{proposition}[Estimation of the Explanation Distribution]
	\label{prop:sampling}
	Let Assumption~\ref{ass:measurable} hold and let $g:\mathbb{R}^N\to\mathbb{R}^k$ be Borel measurable with $\mathbb{E}_{Q_x}\|g(R)\| < \infty$.
	\begin{enumerate}
		\item \textbf{Sampling equivalence.} If $\theta^{(s)}$ are independent draws from $q$ and $R^{(s)} = h_x(\theta^{(s)})$, then $\{R^{(s)}\}$ are independent samples from $Q_x$, and
		\begin{equation}
			\label{eq:mc_equivalence}
			\frac{1}{S}\sum_{s=1}^{S} g(R^{(s)})
			\xrightarrow{\mathrm{a.s.}}
			\mathbb{E}_{Q_x}[g(R)] \qquad (S\to\infty).
		\end{equation}
	Here the convergence is almost sure.
		\item \textbf{Approximation bias and finite-sample error.} Suppose additionally that $P_x$ and $Q_x$ are supported on a common set $\mathcal{R}_x$ of diameter $D_R \defeq \sup_{r,r'\in\mathcal{R}_x}\|r-r'\|_2 < \infty$, and let $g$ be $1$-Lipschitz and scalar-valued. Then for any $\delta\in(0,1)$, with probability at least $1-\delta$,
		\begin{equation}
			\label{eq:bias}
			\left|\frac{1}{S}\sum_{s=1}^{S} g(R^{(s)}) - \mathbb{E}_{P_x}[g(R)]\right|
			\le
			W_1(P_x, Q_x) + D_R\sqrt{\frac{\log(2/\delta)}{2S}},
		\end{equation}
		where $W_1$ is the Wasserstein-1 distance on explanation space.
	\end{enumerate}
\end{proposition}

\begin{proof}
	For part 1, measurable transformations preserve independence. The push-forward identity gives $\mathbb{E}_{\theta\sim q}[g(h_x(\theta))] = \mathbb{E}_{Q_x}[g(R)]$, and the strong law of large numbers applies by the moment assumption \cite{durrett2019probability}. For part 2, $\mathbb{E}_{Q_x}[g(R)]$ is added and subtracted. The Kantorovich--Rubinstein dual form of $W_1$ bounds the first difference \cite{villani2003topics}, and Hoeffding's inequality \cite{hoeffding1963probability} applied to the bounded range $D_R$ bounds the second. A full treatment, including simultaneous coordinate-wise bounds and quantile consistency conditions, is given in \cite{chen2026unified}.
\end{proof}

The bound \eqref{eq:bias} separates two error sources for bounded Lipschitz expectations: the first term measures the discrepancy between the exact and approximate explanation distributions and is independent of $S$. The second is the Monte Carlo error and vanishes at rate $O(S^{-1/2})$. For the full-network linearized prediction described in Section~\ref{sec:laplace}, the first term includes prediction approximation as well as posterior approximation. Section~\ref{sec:sensitivity_exp} examines the practical sensitivity of the summaries to the sample budget.

Setting $g(R)=R$ coordinate-wise in \eqref{eq:mc_equivalence} yields the mean estimate $\bar{R}(x)$. Averaging the squared deviations $(R[n]-\bar{R}[n])^2$ gives the variance estimate $\hat{\sigma}^2[R(x)][n]$. Setting $g(R)=\mathbf{1}[R[n]\le r]$ yields the marginal cumulative distribution function at each time index, from which percentile explanations are obtained (Section~\ref{sec:uq}).

\begin{remark}[Integrability for occlusion sensitivity]
	\label{rem:integrability}
	Each window relevance \eqref{eq:occ_window} is a difference of two softmax probabilities, so $|R_t|\le 1$ uniformly in $\theta$, and $\|h_x(\theta)\|_2 \le \sqrt{N}$. All moments are finite, the common-support diameter satisfies $D_R \le 2\sqrt{N}$, and the mean and variance estimators satisfy the conditions of Proposition~\ref{prop:sampling}. Unlike potentially unbounded attribution scores such as raw gradients, probability-based occlusion scores are bounded by construction, providing a further reason for their use here.
\end{remark}

Each posterior draw $\theta^{(s)}$ simultaneously yields a predictive sample and an explanation sample $R^{(s)}(x)$, coupling predictive and explanation uncertainty through shared samples. The same sampled function is used for the original and all occluded inputs. The out-of-distribution (OOD) analysis of Section~\ref{sec:ood} compares the resulting predictive and explanation statistics.

\subsection{Post-Hoc Posterior via Laplace Approximation}
\label{sec:laplace}
The approximate posterior $q$ is constructed by Laplace approximation (LA). The method operates on classifiers already in service by fitting the posterior to the trained network without retraining or architectural modification. Curvature fitting and posterior selection are performed offline.

Table~\ref{tb:inference} compares alternative inference schemes. Markov chain Monte Carlo (MCMC) obtains posterior samples through iterative likelihood-based simulation. Variational inference (VI) \cite{blundell2015weight} optimizes a parametric posterior within the chosen family and requires additional variational optimization. Deep ensembles \cite{lakshminarayanan2017simple} multiply training cost by the ensemble size and require additional training runs. Monte Carlo dropout (MC-dropout) \cite{gal2016dropout} uses stochastic dropout masks during prediction. The Gaussian approximation based on stochastic weight averaging (SWAG) \cite{maddox2019swag} uses saved parameter iterates from the final training epochs. LA constructs a local Gaussian approximation around the trained parameters using training-data curvature and posterior selection. Section~\ref{sec:bnn_comparison} compares LA with dropout and deep ensembles.

\begin{table}[tbp]
\scriptsize
\centering
\setlength{\tabcolsep}{4pt}
\renewcommand{\arraystretch}{1.15}
\caption{Posterior-inference schemes for an existing classifier. Additional computation covers posterior construction and prediction. Measured costs for the methods evaluated here appear in Section~\ref{sec:timing}.}
\label{tb:inference}
\begin{tabularx}{\linewidth}{@{}p{1.8cm}XXX@{}}
\toprule
Scheme & Applicability to an existing classifier & Additional computation & Posterior approximation \\
\midrule
MCMC & Trained parameters can initialize the chain & Repeated likelihood and gradient evaluations & Asymptotically exact sampling \\
VI \cite{blundell2015weight} & Trained parameters can initialize a variational fit & Variational optimization & Chosen variational family \\
Deep ensemble \cite{lakshminarayanan2017simple} & Additional trained models required & Additional training and member predictions & Discrete ensemble of trained models \\
MC-dropout \cite{gal2016dropout} & Direct for dropout-trained models & Repeated stochastic forward passes & Dropout-induced distribution \\
SWAG \cite{maddox2019swag} & Requires saved training trajectory & Moment estimation, storage and sampled predictions & Trajectory-based Gaussian \\
Full-network LA \cite{ritter2018scalable,daxberger2021laplace} & Direct from trained parameters & Training-data curvature, posterior selection and prediction & Local Gaussian over all parameters \\
Last-layer LA \cite{kristiadi2020bit,daxberger2021laplace} & Direct with a fixed feature extractor & Feature curvature, posterior selection and sampled-head prediction & Conditional Gaussian over last-layer parameters \\
\bottomrule
\end{tabularx}
\end{table}

\paragraph{Laplace posterior and curvature}

In the Bayesian formulation, maximum a posteriori (MAP) estimation minimizes the regularized loss,
\begin{equation}
	\theta_{\text{MAP}} = \arg\min_{\theta}\; \mathcal{L}_{\text{reg}}(\mathcal{D}, \theta) = \arg\min_{\theta}\;\bigl(\mathcal{L}(\mathcal{D}, \theta)+r(\theta)\bigr),
\end{equation}
where the empirical loss and regularizer correspond to the negative log-likelihood and negative log-prior of \eqref{eq:prior}, $\mathcal{L}(\mathcal{D}, \theta) = -\log p(\mathcal{D}\mid\theta)$ and $r(\theta) = -\log p(\theta)$ \cite{daxberger2021laplace}. LA expands $\mathcal{L}_{\text{reg}}$ to second order around $\theta_{\text{MAP}}$ (the first-order term vanishes at the minimizer), giving the Gaussian
\begin{equation}
	\label{eq:laplace}
	q(\theta) = \mathcal{N}\!\left(\theta_{\text{MAP}},\, \Sigma\right),
	\qquad
	\Sigma =
	\Bigl[
	\nabla_{\theta}^{2}\mathcal{L}_{\mathrm{reg}}(\mathcal{D},\theta)\big|_{\theta=\theta_{\mathrm{MAP}}}
	\Bigr]^{-1}.
\end{equation}

For a mean cross-entropy training objective with coupled $L_2$ penalty,
\begin{equation}\label{eq:train_prior}
	\mathcal L_{\rm train}(\theta)=\frac1M\sum_{i=1}^{M}\ell_i(\theta)
	+\frac{\gamma}{2}\|\theta\|^2,
\end{equation}
multiplication by $M$ gives the summed negative-log-posterior objective with reference precision $\lambda_{\rm train}=M\gamma$.

For post-hoc LA, the Gaussian is centered on the fixed trained parameters $\hat\theta$, with its covariance selected as described below \cite{daxberger2021laplace,ritter2018scalable}.

Let $z_\theta(x)$ denote logits, $p_i=\operatorname{softmax}(z_{\hat\theta}(x_i))$, and $J_i=\nabla_\theta z_\theta(x_i)|_{\hat\theta}$. Following scalable LA \cite{ritter2018scalable,daxberger2021laplace}, the likelihood Hessian is replaced by the positive-semidefinite generalized Gauss--Newton (GGN) matrix summed over training instances,
\begin{equation}\label{eq:ggn}
G=\sum_{i=1}^{M}J_i^\top\bigl[\operatorname{diag}(p_i)-p_i p_i^\top\bigr]J_i.
\end{equation}
For the categorical likelihood this is the model Fisher, with the label expectation taken under the model predictive distribution. With $C$ denoting the selected structural approximation to $G$, the posterior precision is
\begin{equation}\label{eq:precision}
\Sigma^{-1}=sC+\lambda I.
\end{equation}
Here $s>0$ scales the data curvature and $\lambda>0$ is the prior precision. The prior term $\lambda I$ regularizes directions of weak curvature. Empirical Hessian studies report a large near-zero spectral bulk and relatively few outlying eigenvalues in overparameterized networks \cite{sagun2018hessian}. Following post-training covariance selection in \cite{ritter2018scalable,kristiadi2020bit}, $\lambda$ can be selected with the center fixed. The value $\lambda_{\rm train}$ serves as the training reference.

\paragraph{Full-network and last-layer prediction}

The GGN omits the likelihood-Hessian terms involving second derivatives of the logits. At $\hat\theta$, it equals the likelihood Hessian of the model obtained by linearizing the logits with respect to the parameters. This gives a generalized linear model with the same categorical likelihood. Following \cite{immer2021linearization}, posterior prediction uses this model,
\begin{equation}\label{eq:glm}
\widetilde z_\theta(x)=z_{\hat\theta}(x)+J_{\hat\theta}(x)(\theta-\hat\theta),
\qquad \widetilde f_\theta(x)=\operatorname{softmax}(\widetilde z_\theta(x)).
\end{equation}
Linearization is with respect to parameters and preserves the deployed logits at $\theta=\hat\theta$ for each original or occluded input. Under $q$, the resulting logit covariance is $J_{\hat\theta}(x)\Sigma J_{\hat\theta}(x)^\top$. The softmax in \eqref{eq:glm} is retained. Writing the resulting occlusion operator as $\widetilde h_x$, the explanation distribution induced by the linearized predictor is $\widetilde Q_x=(\widetilde h_x)_{\#}q$. Proposition~\ref{prop:sampling} applies with $h_x$ and $Q_x$ replaced by $\widetilde h_x$ and $\widetilde Q_x$.

For last-layer LA, write $\theta=(\psi,w)$ and freeze the feature extractor $\phi_{\hat\psi}$:
\begin{equation}\label{eq:llla}
z_w(x)=W\phi_{\hat\psi}(x)+b,\qquad w=(\operatorname{vec}W,b).
\end{equation}
The logits are affine in $w$, so their second derivatives vanish and the last-layer likelihood Hessian equals its GGN. Parameter linearization is exact within this fixed-feature family for original and occluded inputs alike. This motivates last-layer inference in addition to its computational advantage \cite{kristiadi2020bit,daxberger2021laplace}. Diagonal curvature retains only coordinate-wise terms, Kronecker-factored approximate curvature (KFAC) represents blocks through smaller Kronecker factors, and full curvature retains all last-layer parameter couplings \cite{ritter2018scalable}. The importance of the coordinate basis for diagonal curvature approximations is further illustrated by the eigenbasis analysis of \cite{george2018ekfac}.

\paragraph{Application to the deployed PQD classifier}

The full-network posterior uses $C=\operatorname{diag}(G)$ and the linearized predictor in \eqref{eq:glm}. All network parameters, including batch-normalization affine parameters, are included. Running statistics remain fixed in evaluation mode. Diagonal, KFAC and full last-layer curvature are also compared using the same last-layer parameter subset and posterior-selection rule. The deterministic baseline is referred to as MAP.

Posterior width is selected using predictive probabilities. Kristiadi et al.\ \cite{kristiadi2020bit} balance validation predictive fit with high entropy on auxiliary inputs. Following this criterion, the primary objective is
\begin{equation}\label{eq:selection}
\widehat\lambda=\arg\min_\lambda\left\{
-\frac1{n_v}\sum_{i=1}^{n_v}\log\bar p_\lambda(y_i\mid x_i)
-\frac{\beta}{n_o}\sum_{j=1}^{n_o}H[\bar p_\lambda(\cdot\mid u_j)]\right\},
\end{equation}
where $\bar p_\lambda$ is the Monte Carlo mean of predictive probabilities, $H$ is positive Shannon entropy, and $\beta$ weights the auxiliary-input term. The labeled validation set has size $n_v$, and $\{u_j\}_{j=1}^{n_o}$ is a fixed auxiliary-input set. Section~\ref{sec:4.1} describes the data and settings used to select $\lambda$.

Attributions computed from the single MAP estimate are called MAP explanations. Summaries of the explanation distribution at a specified percentile level are called B-explanations.

\subsection{Percentile Bayesian Explanations}
\label{sec:uq}
For each time index $n$, let $F_n$ denote the marginal cumulative distribution function of $R[n]$ under $Q_x$, estimated per Proposition~\ref{prop:sampling}. The $\alpha$-percentile B-explanation is the map
\begin{equation}
	\label{eq:percentile}
	R^{(\alpha)}(x)[n] = \inf\{r : F_n(r) \ge \alpha/100\}, \qquad n=1,\dots,N,
\end{equation}
computed empirically as the $\lceil \alpha S/100 \rceil$-th order statistic of $\{R^{(s)}[n]\}_{s=1}^S$.

The empirical percentile is a consistent estimator of \eqref{eq:percentile} under the quantile-consistency conditions in \cite{chen2026unified}. Moreover, percentile pairs carry a distribution-free finite-sample guarantee: for independent draws from a continuous marginal distribution, the band between the $k$-th and $\ell$-th order statistics contains a further independent explanation draw with probability $(\ell-k)/(S+1)$ \cite{chen2026unified,david2003order}. The band $[R^{(5)}, R^{(95)}]$ summarizes the nominal central $90\%$ of the posterior explanation distribution at each time index. This motivates the levels used throughout: $\alpha \in \{5, 25, 50, 75, 95\}$ comprises the two boundaries of the $90\%$ band, the interquartile boundaries, and the median, giving a five-number summary of the marginal explanation distribution. These levels describe the distribution rather than define a localization-optimal operating point. Section~\ref{sec:percentile_results} examines results across these levels.

A high value of $R^{(\alpha)}(x)[n]$ at a low percentile $\alpha$ indicates that nearly the entire posterior assigns high relevance to sample $n$. High values appearing only at high percentiles indicate relevance supported by only part of the posterior. Low percentiles therefore summarize recurrent attribution, while high percentiles also retain relevance assigned by a smaller fraction of posterior samples. Fig.~\ref{fig:b-occlusion} illustrates this progression, and the per-class analysis in Section~\ref{sec:percentile_results} examines the corresponding patterns. An overview of the complete B-explanation method is provided in Fig.~\ref{fig:B-Occ framework}.

\begin{figure*}[htbp]
	\captionsetup[subfigure]{labelformat=empty}
	\centering
	\begin{subfigure}[b]{0.4\textwidth}
		\centering
		\resizebox{\linewidth}{1.6cm}{%
			\includegraphics[scale=1,width=\textwidth]{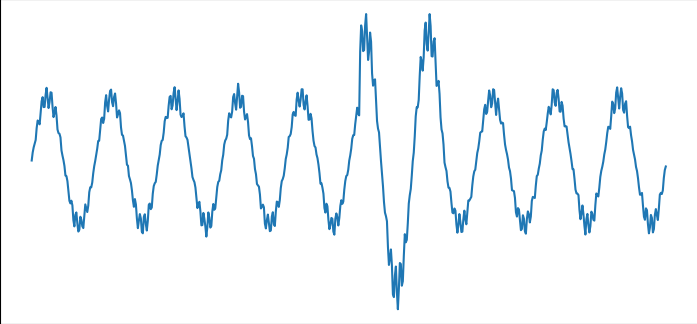}
		}
		\caption{Original Signal}
	\end{subfigure}
	\begin{subfigure}[b]{0.4\textwidth}
		\centering
		\resizebox{\linewidth}{1.6cm}{%
			\includegraphics[scale=1,width=\textwidth]{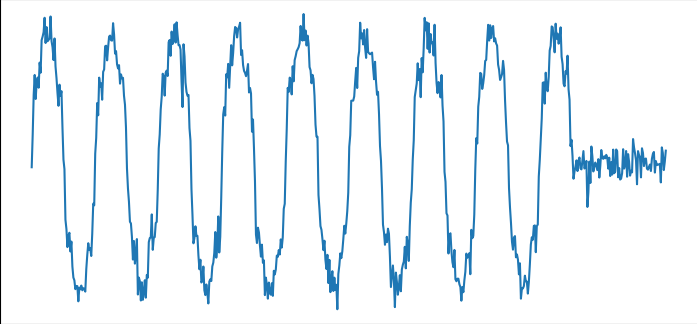}
		}
		\caption{Original Signal}
	\end{subfigure}

	\begin{subfigure}[b]{0.4\textwidth}
		\centering
		\resizebox{\linewidth}{0.9cm}{%
			\includegraphics[scale=1,width=\textwidth]{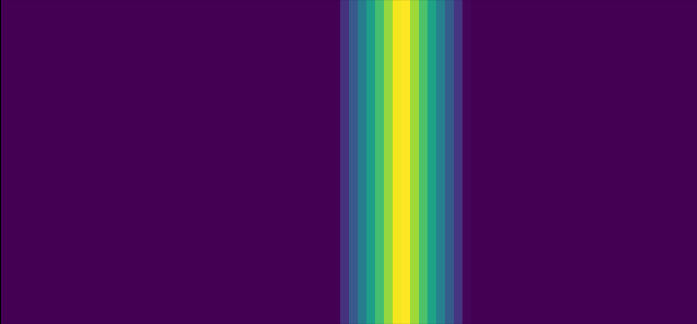}%
		}
		\caption{MAP Explanation}
	\end{subfigure}
	\begin{subfigure}[b]{0.4\textwidth}
		\centering
		\resizebox{\linewidth}{0.9cm}{%
			\includegraphics[scale=1,width=\textwidth]{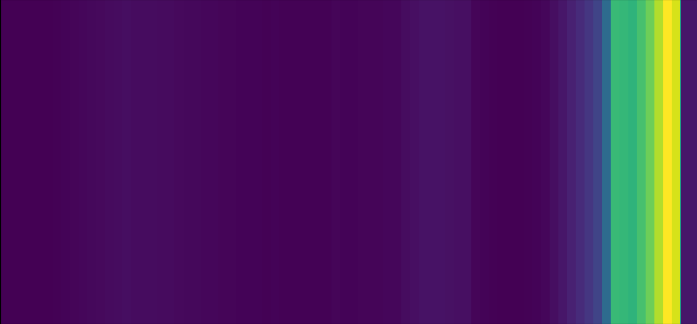}
		}
		\caption{MAP Explanation}
	\end{subfigure}

	\begin{subfigure}[b]{0.4\textwidth}
		\centering
		\resizebox{\linewidth}{0.9cm}{%
			\includegraphics[scale=1,width=\textwidth]{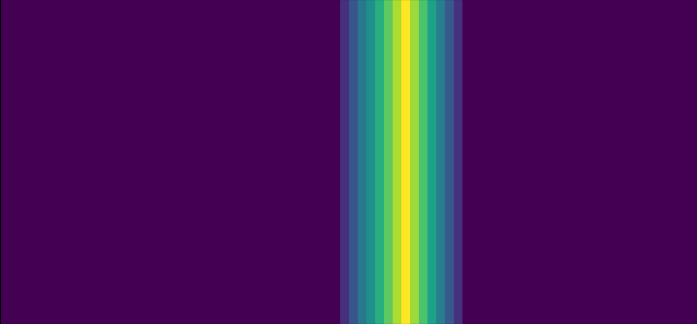}
		}
		\caption{B-Explanation $\alpha=5$ (lower edge, 90\% band)}
	\end{subfigure}
	\begin{subfigure}[b]{0.4\textwidth}
		\centering
		\resizebox{\linewidth}{0.9cm}{%
			\includegraphics[scale=1,width=\textwidth]{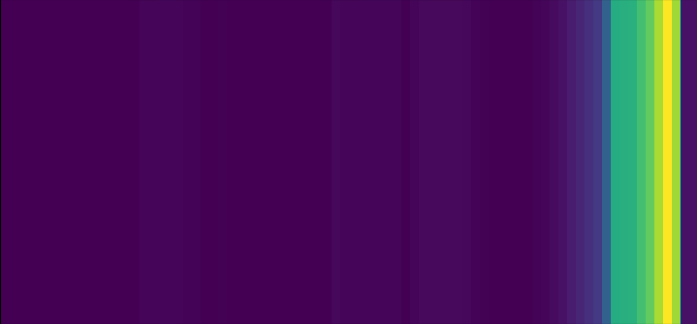}
		}
		\caption{B-Explanation $\alpha=5$ (lower edge, 90\% band)}
	\end{subfigure}

	\begin{subfigure}[b]{0.4\textwidth}
		\centering
		\resizebox{\linewidth}{0.9cm}{%
			\includegraphics[scale=1,width=\textwidth]{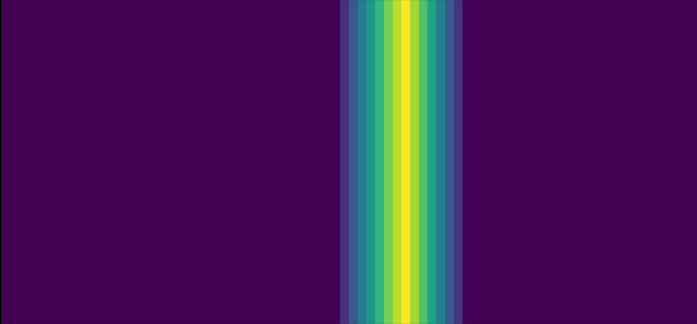}
		}
		\caption{B-Explanation $\alpha=25$ (lower quartile)}
	\end{subfigure}
	\begin{subfigure}[b]{0.4\textwidth}
		\centering
		\resizebox{\linewidth}{0.9cm}{%
			\includegraphics[scale=1,width=\textwidth]{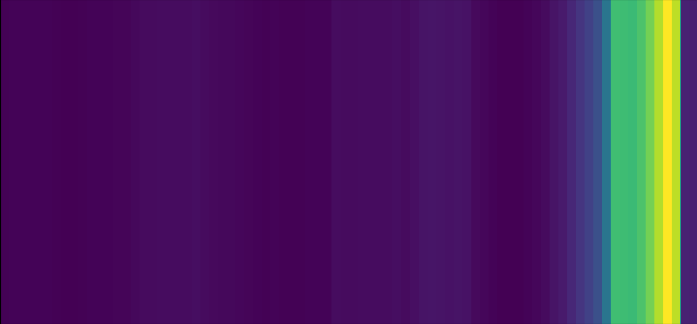}
		}
		\caption{B-Explanation $\alpha=25$ (lower quartile)}
	\end{subfigure}

	\begin{subfigure}[b]{0.4\textwidth}
		\centering
		\resizebox{\linewidth}{0.9cm}{%
			\includegraphics[scale=1,width=\textwidth]{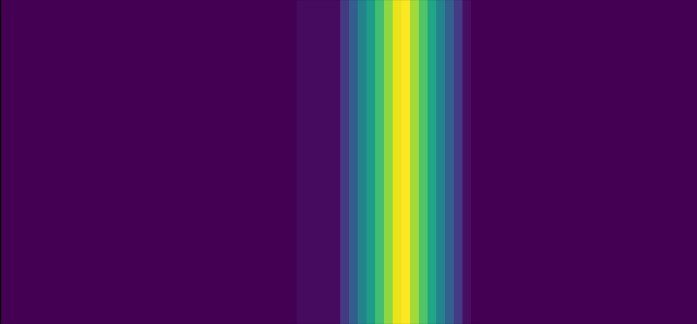}
		}
		\caption{B-Explanation $\alpha=50$ (median)}
	\end{subfigure}
	\begin{subfigure}[b]{0.4\textwidth}
		\centering
		\resizebox{\linewidth}{0.9cm}{%
			\includegraphics[scale=1,width=\textwidth]{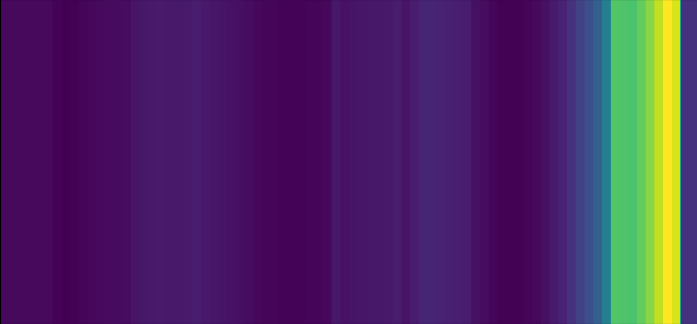}
		}
		\caption{B-Explanation $\alpha=50$ (median)}
	\end{subfigure}

	\begin{subfigure}[b]{0.4\textwidth}
		\centering
		\resizebox{\linewidth}{0.9cm}{%
			\includegraphics[scale=1,width=\textwidth]{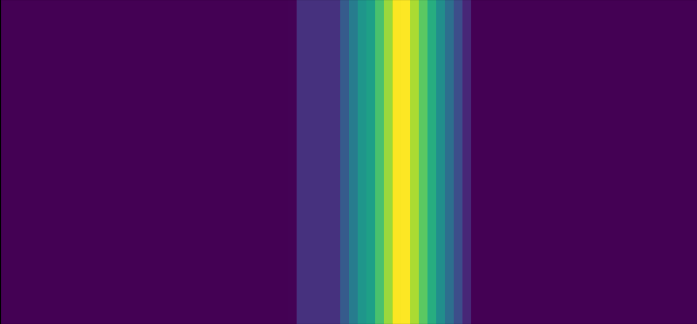}
		}
		\caption{B-Explanation $\alpha=75$ (upper quartile)}
	\end{subfigure}
	\begin{subfigure}[b]{0.4\textwidth}
		\centering
		\resizebox{\linewidth}{0.9cm}{%
			\includegraphics[scale=1,width=\textwidth]{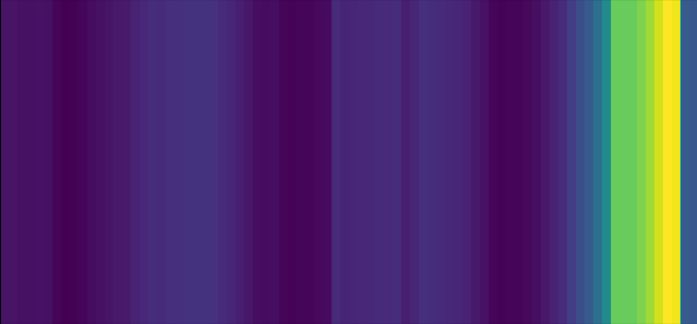}
		}
		\caption{B-Explanation $\alpha=75$ (upper quartile)}
	\end{subfigure}

	\begin{subfigure}[b]{0.4\textwidth}
		\centering
		\resizebox{\linewidth}{0.9cm}{%
			\includegraphics[scale=1,width=\textwidth]{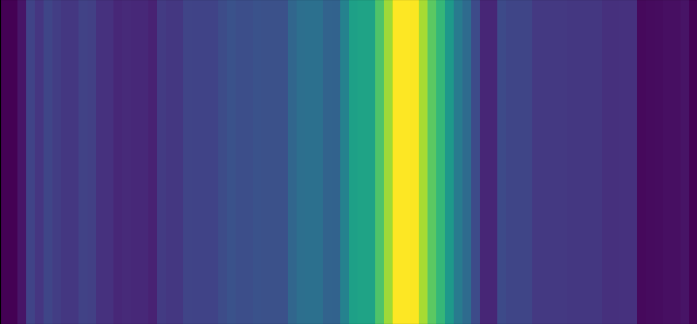}
		}
		\caption{B-Explanation $\alpha=95$ (upper edge, 90\% band) \\ (a) Flicker+Swell}
	\end{subfigure}
	\begin{subfigure}[b]{0.4\textwidth}
		\centering
		\resizebox{\linewidth}{0.9cm}{%
			\includegraphics[scale=1,width=\textwidth]{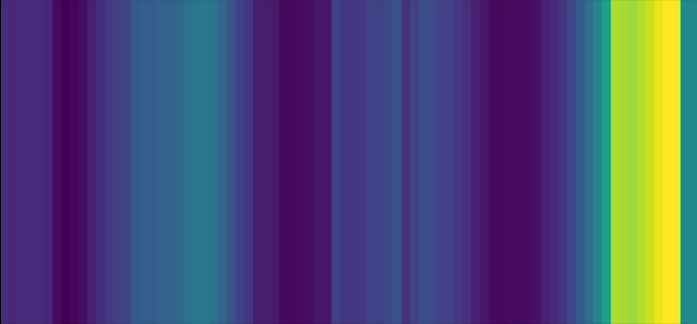}
		}
		\caption{B-Explanation $\alpha=95$ (upper edge, 90\% band) \\ (b) Sag+Harmonics}
	\end{subfigure}

	\begin{center}
		\begin{subfigure}[b]{0.4\textwidth}
			\centering
			\resizebox{\linewidth}{0.45cm}{%
				\includegraphics[scale=1,width=\textwidth]{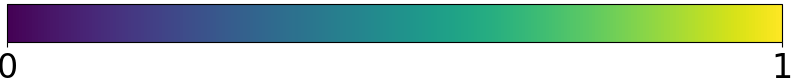}
			}
			\caption{(c) Colormap}
		\end{subfigure}
	\end{center}
	\caption{Uncertainty quantification of PQD classifier explanations with min-max normalized relevance scores. Rows show the MAP explanation followed by percentile B-explanations at $\alpha \in \{5,25,50,75,95\}$. The pair $\alpha=5$ and $\alpha=95$ bounds a pointwise posterior summary band, and $\alpha=25/50/75$ are the quartile and median summaries. Experimental settings are described in Sections~\ref{sec:3.1} and \ref{sec:4.1}.}
	\label{fig:b-occlusion}
\end{figure*}

\begin{figure*}[hbt]
	\centering
	\includegraphics[scale=1,width=\textwidth]{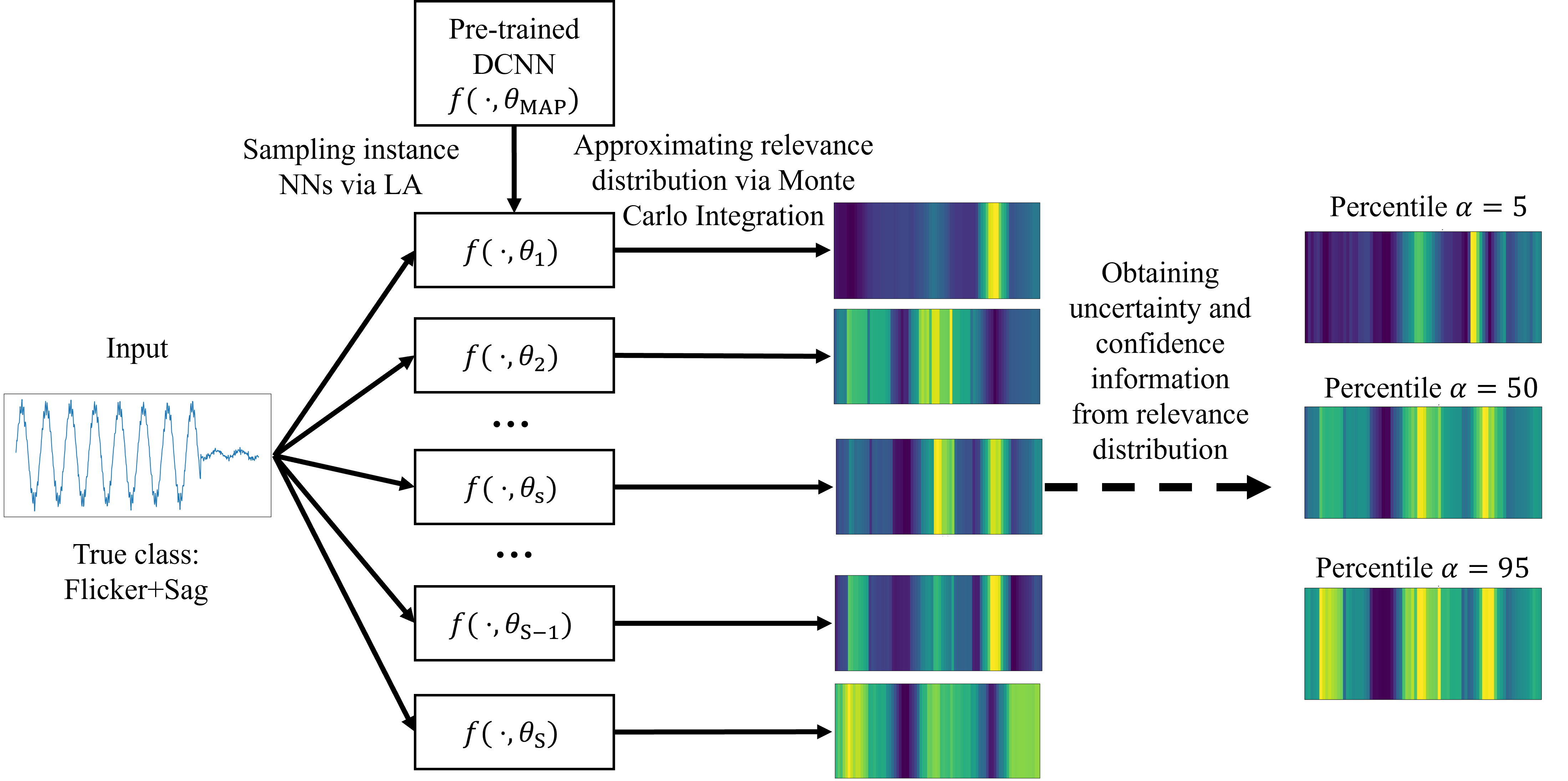}
	\caption{Overview of the post-hoc B-explanation method for a deployed PQD classifier.}
	\label{fig:B-Occ framework}
\end{figure*}

\subsection{Multimodality of Explanations}
\label{sec:cluster}
\begin{figure*}[htbp]
	\captionsetup[subfigure]{labelformat=empty}
	\centering
	\begin{subfigure}[b]{0.202\textwidth}
		\centering
		\resizebox{\linewidth}{3.2cm}{%
			\includegraphics[scale=1,width=\textwidth]{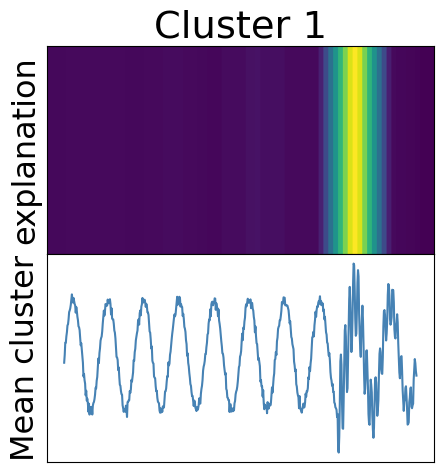}
		}
	\end{subfigure}
	\begin{subfigure}[b]{0.186\textwidth}
		\centering
		\resizebox{\linewidth}{3.2cm}{%
			\includegraphics[scale=1,width=\textwidth]{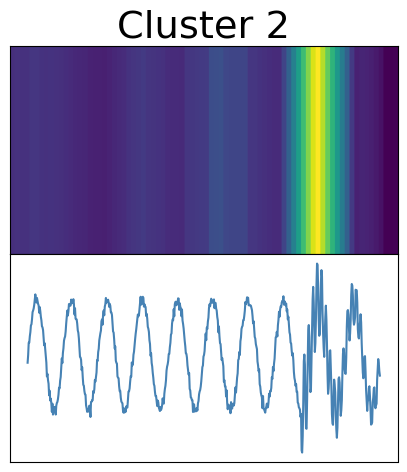}}
	\end{subfigure}
	\begin{subfigure}[b]{0.186\textwidth}
		\centering
		\resizebox{\linewidth}{3.2cm}{%
			\includegraphics[scale=1,width=\textwidth]{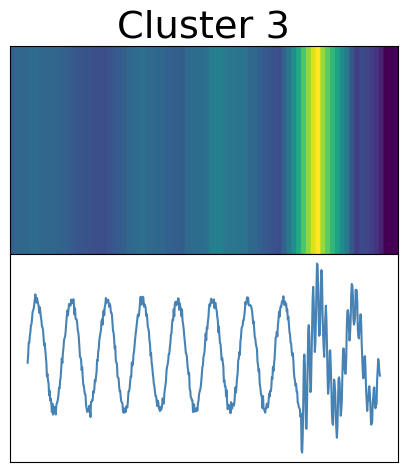}}
	\end{subfigure}
	\begin{subfigure}[b]{0.186\textwidth}
		\centering
		\resizebox{\linewidth}{3.2cm}{%
			\includegraphics[scale=1,width=\textwidth]{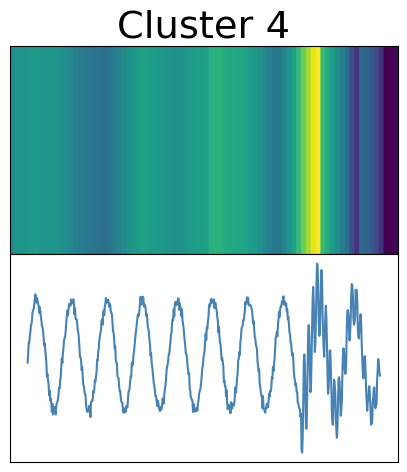}}
	\end{subfigure}
	\begin{subfigure}[b]{0.186\textwidth}
		\centering
		\resizebox{\linewidth}{3.2cm}{%
			\includegraphics[scale=1,width=\textwidth]{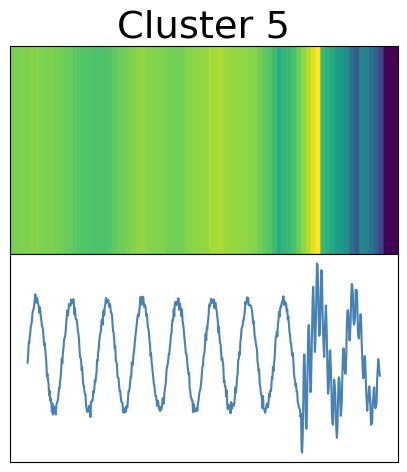}}
	\end{subfigure}

	\begin{subfigure}[b]{0.204\textwidth}
		\centering
		\resizebox{\linewidth}{2.5cm}{%
			\includegraphics[scale=1,width=\textwidth]{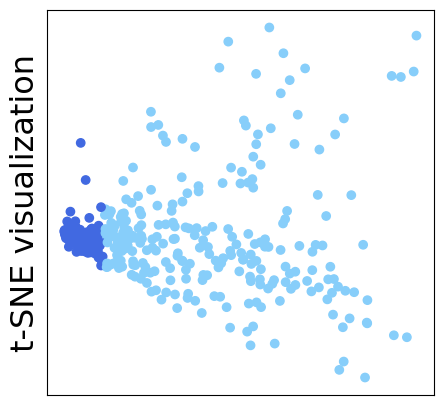}}
	\end{subfigure}
	\begin{subfigure}[b]{0.186\textwidth}
		\centering
		\resizebox{\linewidth}{2.5cm}{%
			\includegraphics[scale=1,width=\textwidth]{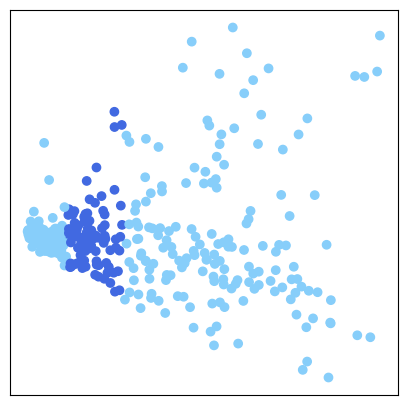}}
	\end{subfigure}
	\begin{subfigure}[b]{0.186\textwidth}
		\centering
		\resizebox{\linewidth}{2.5cm}{%
			\includegraphics[scale=1,width=\textwidth]{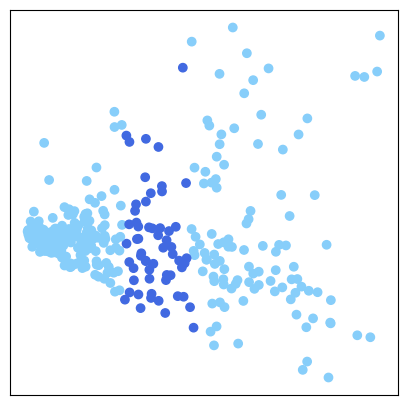}}
	\end{subfigure}
	\begin{subfigure}[b]{0.186\textwidth}
		\centering
		\resizebox{\linewidth}{2.5cm}{%
			\includegraphics[scale=1,width=\textwidth]{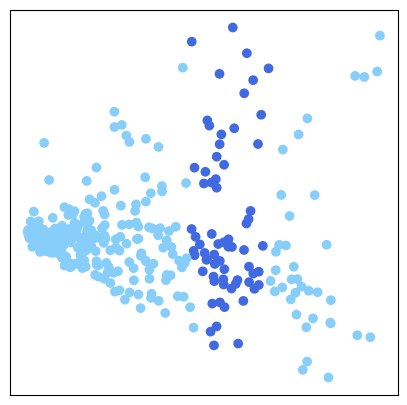}}
	\end{subfigure}
	\begin{subfigure}[b]{0.186\textwidth}
		\centering
		\resizebox{\linewidth}{2.5cm}{%
			\includegraphics[scale=1,width=\textwidth]{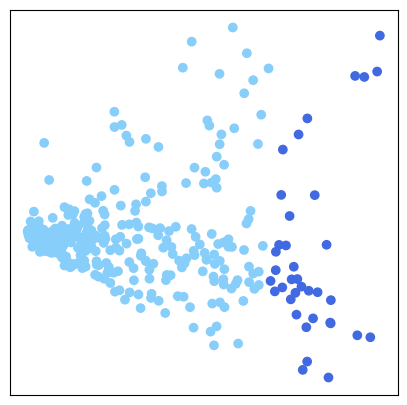}}
	\end{subfigure}
	\caption{Visualization of explanation multimodality in PQD classifiers (disturbance: oscillatory transient). The first row presents the mean explanation for each cluster. The second row shows the t-distributed stochastic neighbor embedding (t-SNE) projection of the explanation distribution with highlighted cluster assignments. Experimental settings are described in Sections~\ref{sec:3.1} and \ref{sec:4.1}.}
	\label{fig:cluster}
\end{figure*}

Percentile maps summarize the explanation distribution marginally, per time index. Following \cite{bykov2021explaining}, sampled relevance vectors are clustered to reveal groups of explanations, i.e., internally consistent attribution patterns that differ between groups of posterior samples. In Fig.~\ref{fig:cluster}, 500 sampled explanations for an oscillatory-transient signal partition into five $k$-means clusters, visualized with t-distributed stochastic neighbor embedding (t-SNE) \cite{wani2025comprehensive}. All corresponding models classify the input correctly, yet the dominant relevance shifts gradually away from the true disturbance region across clusters, and the tightest cluster concentrates on it. These clusters show that a single-mode Laplace posterior can produce distinct attribution patterns that percentile maps summarize pointwise.

\subsection{Algorithm and Computational Complexity}
\label{sec:complexity}
Algorithm~\ref{alg:bexp} summarizes the complete procedure. Table~\ref{tb:variables} defines its inputs, intermediate quantities and outputs. The trained parameters of the deployed classifier are denoted by $\hat\theta$. Posterior prediction uses either the full-network linearized model or a sampled last-layer classifier with a fixed feature extractor, as specified in Section~\ref{sec:laplace}.

\begin{table}[tbp]

	\scriptsize
	\centering
	\caption{Variables in Algorithm~\ref{alg:bexp}.}
	\label{tb:variables}
	\begin{tabular}{llp{6.2cm}}
		\toprule
		Symbol & Role & Meaning \\
		\midrule
		$\hat\theta$ & input & trained parameters of the deployed classifier \\
		$\mathcal{D}$ & input & training set (for curvature accumulation) \\
		$\mathcal D_v,\mathcal U$ & input & validation set and fixed auxiliary-input set \\
$C$ & intermediate quantity & selected GGN curvature approximation \\
$x$ & input & waveform to be explained, length $N$ \\
		$c$ & input & class of interest \\
		$w, v, b$ & input & occlusion window, stride, baseline value \\
		$S$ & input & number of posterior samples \\
		$\mathcal{A}$ & input & percentile levels, default $\{5,25,50,75,95\}$ \\
		$\lambda$ & selected hyperparameter & prior precision in \eqref{eq:precision} \\
		$s$ & fixed hyperparameter & curvature scale in \eqref{eq:precision}, set to $1$ \\
		$q$ & intermediate quantity & Laplace posterior $\mathcal{N}(\hat\theta, \Sigma)$ \\
		$R^{(s)}$ & intermediate quantity & normalized occlusion map of posterior sample $s$ \\
		$\{R^{(\alpha)}\}_{\alpha\in\mathcal{A}}$ & output & percentile B-explanations \\
		$\bar{R}, \hat{\sigma}^2$ & output & mean explanation and explanation variance \\
		$\{p^{(s)}\}_{s=1}^S$ & output & predictive samples \\
		\bottomrule
	\end{tabular}
\end{table}

\begin{algorithm}[t]

\caption{Post-hoc B-explanation generation}\label{alg:bexp}
\begin{algorithmic}[1]

\Require $\hat\theta$, $\mathcal D$, $\mathcal D_v$, $\mathcal U$, curvature choice, $x$, $c$, $w$, $v$, $b$, $S$, $\mathcal A$, $s=1$
\Ensure $\{R^{(\alpha)}\}_{\alpha\in\mathcal A}$, $\bar R$, $\hat\sigma^2$, predictive samples
\State \textbf{(offline, once per model)} accumulate the selected GGN approximation $C$ over $\mathcal D$ at $\hat\theta$
\State Select $\lambda$ by the fixed criterion and set $\Sigma\gets(sC+\lambda I)^{-1}$
\State Form the original and $T$ occluded inputs, $T=\lfloor(N-w)/v\rfloor+1$
\State Compute and retain features for last-layer inference, or baseline logits for full-network linearized prediction
\For{$s'=1$ to $S$} \Comment{online, per input $x$}
\State sample $\theta^{(s')}\sim\mathcal N(\hat\theta,\Sigma)$
\State $p^{(s')}\gets$ sampled predictive function evaluated at $x$
\For{$t=0$ to $T-1$}
\State Evaluate the same sampled function on $\tilde x^{(t)}$
\State $R_t\gets[p^{(s')}]_c-$ its probability for class $c$ at $\tilde x^{(t)}$
\EndFor
\State $R^{(s')}[n]\gets\frac{1}{|T_n|}\sum_{t\in T_n}R_t$ for all $n$ \Comment{Eq.~\eqref{eq:occ_avg}}
\State Apply absolute value and min-max normalization to $R^{(s')}$ as in Section~\ref{sec:3.1}
\EndFor
\State $\bar R\gets\frac1S\sum_{s'}R^{(s')}$,\quad $\hat\sigma^2[n]\gets\frac1S\sum_{s'}(R^{(s')}[n]-\bar R[n])^2$
\For{$\alpha\in\mathcal A$}
\State $R^{(\alpha)}[n]\gets$ $\lceil\alpha S/100\rceil$-th order statistic of $\{R^{(s')}[n]\}_{s'}$ for all $n$
\EndFor
\end{algorithmic}
\end{algorithm}

For $N=640$, $w=64$ and $v=8$, there are $T=73$ occlusion positions. Each sampled function is therefore evaluated on the original waveform and its 73 occluded versions, giving $T+1=74$ waveform evaluations per posterior draw. Full-network and last-layer inference perform these evaluations differently. In the full-network method, the trained logits are computed for each input and each posterior draw supplies the correction $J_{\hat\theta}(x)(\theta^{(s)}-\hat\theta)$ in \eqref{eq:glm}. This correction is computed by a Jacobian-vector product through the network. In last-layer inference, the feature extractor is fixed, so the features of the original and occluded inputs are computed once and reused for all sampled final layers. The logits are affine in the last-layer parameters, making parameter linearization exact within this family.

Let $C_f$ denote the cost of evaluating the trained network on one input, $C_J$ the cost of one full-network Jacobian-vector product, $C_\phi$ the feature-extraction cost and $C_h$ the cost of evaluating a sampled final layer. Let $C_{\rm draw}$ denote the cost of generating one parameter sample from the chosen Gaussian posterior. The online costs are
\begin{equation}\label{eq:complexity}
\begin{aligned}
\text{full network:}\quad& O((T+1)C_f+SC_{\rm draw}+S(T+1)C_J+A),\\
\text{last layer:}\quad& O((T+1)C_\phi+SC_{\rm draw}+S(T+1)C_h+A),\\
& A=STN+SN\log S.
\end{aligned}
\end{equation}
Here $A$ includes window-to-sample averaging and percentile sorting. For $S=100$, both methods perform $100\times74=7{,}400$ evaluations across posterior draws and waveform versions. In the full-network method, each evaluation requires a network derivative calculation. In the last-layer method, it requires only evaluation of the sampled final layer on the reused features. The feature-extraction cost is therefore independent of $S$ for last-layer inference.

The covariance structure determines which parameter correlations are retained and the cost of storing and sampling the posterior. Diagonal LA stores one variance per selected parameter and scales Gaussian samples coordinate-wise. Full last-layer LA stores a dense matrix whose size is quadratic in the number of last-layer parameters and uses its matrix factor to generate correlated samples. KFAC represents the layer curvature with smaller Kronecker factors, reducing matrix storage and the associated factorization and sampling operations relative to the dense representation. All three last-layer variants use the same fixed features and sampled-final-layer prediction described above. Thus, choosing the parameter subset determines which part of the network varies, linearization determines how parameter variation changes predictions, and the covariance structure determines the distribution of that variation.

Curvature fitting, posterior selection and preparation of Gaussian sampling factors are performed once per model. These offline costs are reported separately from the per-input costs in \eqref{eq:complexity}. Measured times are reported in Section~\ref{sec:timing}.

\section{Evaluation Framework}
\label{sec:2}
The framework evaluates explanations from three perspectives: localization against physical ground truth, faithfulness to the classifier, and statistical significance of comparisons.

\subsection{Qualitative Evaluation}
\label{sec:3.1}
Relevance vectors are converted to absolute values, normalized to $[0, 1]$ by min-max normalization and visualized with the perceptually uniform `viridis' colormap (Fig.~\ref{fig:b-occlusion}). These steps are applied separately to each sampled map before computing the distribution summaries, and deterministic maps are processed identically. Constant absolute-valued maps are assigned zero relevance after normalization.

\subsection{Localization Metrics}
\label{sec:loc_metrics}
Physical ground truth permits localization assessment in this synthetic PQD setting, addressing a common limitation of XAI evaluation \cite{nauta2023anecdotal}. Disturbance regions are defined against the nominal reference waveform
\begin{equation}
	x_0[n] = A\sin(\omega n+\phi) + \sigma[n],
\end{equation}
where $\sigma$ denotes measurement noise if the reference is observed rather than generated.

\begin{definition}[$\epsilon$-Ground-Truth Disturbance Segmentation]
	\label{def:gt_mask}
	Let $\epsilon \geq 0$ be a threshold representing the expected noise level of the reference. The $\epsilon$-ground-truth disturbance mask is
	\begin{equation}
		R_{GT}(x)[n] = \begin{cases}
			1, & \text{if }|x[n]-x_0[n]|>\epsilon, \\
			0, & \text{otherwise,}
		\end{cases} \quad \forall n \in\{1,\ldots, N\},
	\end{equation}
	and $\mathcal{I}(x) = \{n : R_{GT}(x)[n]=1\}$ with $L = |\mathcal{I}(x)|$. When the generator provides the exact noise-free reference, as in this study, the mask is computed exactly with $\sigma = 0$ and $\epsilon = 0$.
\end{definition}

The absolute-value and normalization steps of Section~\ref{sec:3.1} precede metric computation, yielding the nonnegative maps assumed below. Zero summaries are handled as described in Section~\ref{sec:stats}. Two complementary metrics measure alignment between a nonnegative relevance map $R$ and the mask:
\begin{itemize}
	\item \textbf{Relevance Mass Accuracy (RMA)} \cite{arras2003ground}: the proportion of total relevance inside the disturbance region,
	\begin{equation}
		\text{RMA}(R,x) = \frac{\sum_{n \in \mathcal{I}(x)} R[n]}{\sum_{n=1}^{N} R[n]} .
	\end{equation}
	\item \textbf{Intersection over Union (IoU)} \cite{machlev2021measuring,8953982}: temporal overlap after top-$L$ binarization. With $\top_L(R)$ the indices of the $L$ largest values of $R$ and $\hat{R}$ the corresponding binary mask,
	\begin{equation}
		\text{IoU}(R,x) = \frac{|\top_L(R) \cap \mathcal{I}(x)|}{|\top_L(R) \cup \mathcal{I}(x)|} .
	\end{equation}
	Fixing the predicted-mask cardinality to $L$ allows overlap to be compared without selecting an attribution threshold.
\end{itemize}
Both metrics lie in $[0,1]$, with higher values indicating greater alignment with the disturbance mask. RMA is sensitive to the distribution of relevance mass without binarization. IoU measures overlap after binarization using the ground-truth mask size $L$.

\subsection{Faithfulness Metrics}
\label{sec:faith_metrics}
Localization metrics compare explanations to physical ground truth, whereas faithfulness metrics assess agreement with a specified classifier under perturbation, and enable comparison with attribution methods that do not target localization. Two standard measures are adopted \cite{petsiuk2018rise,samek2016evaluating}. The deletion score replaces non-overlapping 64-sample input blocks by zero in decreasing order of block-mean relevance and records the class probability as a function of the fraction deleted. The area under the curve (AUC) is reported as deletion AUC, and lower values indicate a more faithful explanation. AUC is computed by the trapezoidal rule, and ties are ordered by block index and counted. The insertion score starts from the baseline signal and restores the same blocks in decreasing order of block-mean relevance. Higher insertion AUC indicates a more faithful explanation. In addition, monotonicity is reported as the Spearman rank correlation between the mean relevance assigned to each 64-sample occlusion window at stride eight and the probability drop caused by occluding that window alone. Monotonicity shares the single-window occlusion operator used to construct the explanations and may therefore favor occlusion-based methods. Deletion and insertion use the same zero baseline along sequential perturbation paths. The fixed deployed classifier scores the main comparisons with $c=\hat y$ fixed from its original input. Trained MC-dropout is compared with its own dropout-disabled classifier using its own fixed predicted class. These metrics permit a comparison of B-explanations with LIME \cite{ribeiro2016trust} and SHAP \cite{lundberg2017unified} in Section~\ref{sec:xai_comparison}.

\subsection{Statistical Protocol}
\label{sec:stats}
Per-class means and their class-macro average are reported. The number of evaluated signals accompanies each summary. Paired effects are computed on matched signals with available scores. The overlapping evaluation subsets described in Section~\ref{sec:4.1} are used only for descriptive split sensitivity. The primary test is a two-sided paired $t$ test, with Student-$t$ 95\% confidence intervals (CIs) for the mean difference and Cohen's $d_z$ (mean difference divided by the sample standard deviation of differences). Holm correction is applied separately within each family defined by metric, attribution target, and comparison. Differences are described as improvements only when the corrected test rejects at the $0.05$ level. Displayed intervals are unadjusted. Wilcoxon signed-rank tests are supplementary descriptive checks. Paired tests are not reported when fewer than ten matched signals have valid scores. Zero normalized summaries are counted separately, and their localization and faithfulness scores are omitted from the numerical comparisons. Min-max normalization maps a constant nonzero response to zero. A lower-percentile summary can also be zero throughout when zero-valued regions differ across sampled maps.

Confidence intervals for the area under the receiver operating characteristic curve (AUROC) in OOD detection use 1,000 bootstrap resamples. Matched clean and shifted signals are resampled jointly, while field windows are resampled in their source-recording clusters.

\section{Numerical Results and Discussion}
\label{sec:3}

\subsection{Experimental Setup}
\label{sec:4.1}
We trained the DCNN in Table~\ref{tb:DCNN} for 100 epochs using cross-entropy loss, Adam, and an $L_2$ penalty of $\gamma=10^{-4}$. The initial learning rate was $0.01$ and was halved every 10 epochs. Model selection used validation performance. Occlusion sensitivity uses window $w=64$ (one nominal cycle) and stride $v=8$ with a zero occlusion baseline ($b=0$), per the window-geometry discussion of Section~\ref{sec: occ}. Attribution is computed both for the ground-truth class ($c=y$), to evaluate recovery of the labeled disturbance, and for the fixed prediction of the unperturbed deployed classifier ($c=\hat y$), to explain its decision. $S=100$ posterior samples are drawn unless stated otherwise. The deployed classifier is one of the five ensemble members. A separate model is trained with dropout probability 0.2, and its dropout-disabled predictions provide the matched deterministic control for trained MC-dropout.

Experiments use a synthetic dataset covering the 15 PQD types of \cite{wang2019novel} plus an undisturbed class, generated with the open-source PQD signal generator of \cite{machlev2021open} in its noise-free mode together with the clean reference waveform of each signal. Signals are sampled at 3.2~kHz with a 50~Hz nominal frequency over 10 cycles ($N=640$). Noise-free generation makes the ground-truth disturbance masks exact ($\epsilon=0$ in Definition~\ref{def:gt_mask}). Robustness to measurement noise is evaluated separately by controlled injection in Section~\ref{sec:noise_exp}. The fixed partitions comprise 11,800 training, 1,200 validation and 3,000 test waveforms, with no identical waveform shared across partitions. The test pool contains 200 signals in each disturbance class. The undisturbed class is included only in training. The held-out evaluation therefore covers the 15 disturbance classes. Localization metrics are computed on the 15 disturbance classes only, since the undisturbed class has an empty ground-truth mask. The MAP model and all ensemble members are trained on the training partition under the protocol above, differing only in random seed. Five overlapping evaluation subsets are drawn from the test pool. The 150-signal budget and faithfulness subset has ten fixed signals per class and is drawn from the same test pool. Real-world evaluation uses the University of Cadiz recordings \cite{h2k88d-17} (Section~\ref{sec:realworld}). The field dataset contains 104 windows from 26 source recordings.

Curvature is estimated using the GGN at the trained parameters, with batch-normalization statistics held fixed. The full-network and last-layer posterior dimensions are 164,464 and 2,064, respectively. With $M=11{,}800$, the training-penalty reference is $M\gamma=1.18$. The primary selection rule minimizes validation negative log-likelihood (NLL) with OOD entropy regularization, following Kristiadi et al.\ \cite{kristiadi2020bit}. This criterion \eqref{eq:selection} uses $\beta=0.25$, up to 40 validation signals per class (600 total), and a fixed set of 512 independently generated uniform-noise waveforms, separate from the test data. Each auxiliary sample is drawn pointwise from $U[-a,a]$, $a=\sqrt3\,r_0\approx1.2247$, where $r_0$ is the root-mean-square amplitude of the nominal training reference. This choice matches the expected mean-square amplitude of the noise to that of the reference.

The curvature scale is fixed at $s=1$ to retain the data scale of the summed GGN in \eqref{eq:ggn}. For $\lambda$, a 13-point logarithmic grid spans $1.18\times10^{-3}$ to $1.18\times10^{3}$ at half-decade spacing. The primary selection rule uses 20 posterior draws per candidate. The three candidates with the lowest objective values are re-evaluated using 40 draws. All four selected values lie inside the search interval. As a control, all four posterior variants are also evaluated with $\lambda$ selected by marginal likelihood maximization (empirical Bayes, EB) \cite{daxberger2021laplace}. The Laplace approximation to the training-data marginal likelihood is evaluated on the same grid, holding the trained parameters, data curvature and $s=1$ fixed. This criterion uses the training data rather than validation and auxiliary OOD inputs. Posterior selection is completed before test-set explanation evaluation. Table~\ref{tb:posterior_settings} reports both rules. The tables use the primary selection rule for LA unless a row is marked EB (empirical Bayes). In the tables, All and LL denote full-network and last-layer inference, and diag and full denote diagonal and full covariance.

\clearpage
Table~\ref{tb:accuracy} reports predictive statistics on the full test pool. Let $p^{(s)}(x)$ denote the class-probability vector from sample $s$ and $\bar p(x)=S^{-1}\sum_{s=1}^{S}p^{(s)}(x)$. Predictive entropy is $H(\bar p(x))$, where $H(p)=-\sum_{k=1}^{K}p_k\log p_k$. Mutual information (MI) between the predicted class and the model parameters is estimated by
\[
\mathrm{MI}(x)\approx H(\bar p(x))-\frac{1}{S}\sum_{s=1}^{S}H\bigl(p^{(s)}(x)\bigr).
\]
MI measures disagreement among the sampled predictive distributions. Entropy and MI are averaged across test signals and reported in nats.
\begin{table}[!htbp]

\scriptsize\centering
\setlength{\tabcolsep}{4pt}
\renewcommand{\arraystretch}{1.12}
\caption{Post-hoc precision selection with $s=1$. The primary selection rule minimizes validation NLL with OOD entropy regularization. Empirical Bayes maximizes the approximate training-data marginal likelihood. The training-penalty reference is $\lambda_{\rm train}=1.18$.}
\label{tb:posterior_settings}
\begin{tabular}{lrr}
\toprule
Posterior & $\lambda$: primary selection rule & $\lambda$: empirical Bayes \\
\midrule
All/diag & 118.000 & 37.315 \\
LL/diag & 3.731 & 3.731 \\
LL/KFAC & 1.180 & 3.731 \\
LL/full & 3.731 & 1.180 \\
\bottomrule
\end{tabular}
\end{table}

\begin{table}[!htbp]

\scriptsize\centering
\setlength{\tabcolsep}{4pt}
\renewcommand{\arraystretch}{1.12}
\caption{Predictive statistics on 3,000 test signals. Accuracy uses the Monte Carlo mean prediction, not the average accuracy of individual draws. Entropy and mutual information (MI) are in nats. LA uses the primary selection rule unless marked EB (empirical Bayes).}
\label{tb:accuracy}
\begin{tabular}{lrrr}
\toprule
Method & Accuracy & Entropy & MI \\
\midrule
MAP & 0.9937 & 0.015 & 0.000 \\
All/diag & 0.9943 & 0.716 & 0.517 \\
LL/diag & 0.9937 & 0.342 & 0.203 \\
LL/KFAC & 0.9937 & 0.290 & 0.169 \\
LL/full & 0.9933 & 0.409 & 0.249 \\
All/diag (EB) & 0.9950 & 1.209 & 0.948 \\
LL/diag (EB) & 0.9937 & 0.344 & 0.204 \\
LL/KFAC (EB) & 0.9937 & 0.111 & 0.049 \\
LL/full (EB) & 0.9933 & 1.468 & 1.145 \\
Ensemble-5 & 0.9930 & 0.021 & 0.004 \\
Post-hoc dropout & 0.9940 & 0.267 & 0.143 \\
Dropout control & 0.9877 & 0.025 & 0.000 \\
Trained MC-dropout & 0.9877 & 0.036 & 0.004 \\
\bottomrule
\end{tabular}
\end{table}

\subsection{Percentile B-Explanations: Localization Performance}
\label{sec:percentile_results}
RMA and IoU scores for each disturbance type are computed on the complete 3,000-signal test pool for MAP and percentile B-explanations (Tables \ref{tb:rma} and \ref{tb:iou}), using the primary full-network diagonal posterior. Means are reported over valid maps.

At the aggregate level, the fifth-percentile ($q_{05}$) summary increases class-macro RMA from 0.6408 to 0.6556, while IoU decreases by 0.0341 from the MAP value of 0.6775. Both paired changes are significant after Holm correction (\ref{app:paired}). A larger proportion of relevance lies within the disturbance region, but the overlap of the selected temporal regions decreases.

RMA describes how percentile summaries redistribute attribution mass, while IoU provides a complementary measure of temporal overlap between the highest-ranked locations and the disturbance region. The top-$L$ construction also connects IoU to relevance rank accuracy (RRA) \cite{arras2003ground}, which measures the fraction of the $L$ highest-ranked locations inside the ground-truth region. For an individual signal, if $t$ locations are shared by the two masks of size $L$, then $\mathrm{RRA}=t/L$ and $\mathrm{IoU}=t/(2L-t)=\mathrm{RRA}/(2-\mathrm{RRA})$. Thus, IoU preserves the per-signal ordering of RRA and complements the mass-based comparison provided by RMA.

The percentile level changes this balance. Class-macro RMA is highest at $q_{05}$. The increase over MAP is smaller at $q_{25}$, while the median and upper-percentile summaries fall below MAP. IoU is closer to MAP at $q_{50}$ and $q_{75}$ than at $q_{05}$, although all five percentile summaries have lower IoU than MAP. Both attribution targets give similar clean-data averages, consistent with the deployed classifier's 99.37\% test accuracy. \ref{app:complete} reports all posterior variants, both selection rules and all six summaries for both attribution targets.

\begin{table}[!htbp]

\scriptsize\centering
\setlength{\tabcolsep}{4pt}
\renewcommand{\arraystretch}{1.12}
\caption{RMA by disturbance class for MAP and full-network diagonal LA under the primary selection rule ($c=y$, $S=100$). $^{\ast}$ marks 199 valid maps rather than 200. The remaining class entries use 200. Classes follow the same physical grouping in both tables.}
\label{tb:rma}
\begin{tabular}{lrrrrrr}
\toprule
Class & MAP & $q_{05}$ & $q_{25}$ & $q_{50}$ & $q_{75}$ & $q_{95}$ \\
\midrule
Class macro & 0.641 & 0.656 & 0.643 & 0.624 & 0.609 & 0.607 \\
\midrule
\multicolumn{7}{l}{Amplitude changes and oscillatory transient} \\
Sag & 0.363 & 0.323 & 0.330 & 0.299 & 0.281 & 0.320 \\
Swell & 0.437 & 0.627 & 0.523 & 0.435 & 0.391 & 0.363 \\
Interruption & 0.125 & 0.078 & 0.083 & 0.105 & 0.137 & 0.187 \\
Oscillatory transient & 0.617 & 0.729 & 0.643 & 0.451 & 0.267 & 0.180 \\
\addlinespace
\multicolumn{7}{l}{Short or periodically sparse events} \\
Impulsive transient & 0.010 & 0.008 & 0.011 & 0.010 & 0.008 & 0.008 \\
Notch & 0.034 & 0.038 & 0.033 & 0.034 & 0.033 & 0.033 \\
Spike & 0.031 & 0.031 & 0.031 & 0.031 & 0.031 & 0.031 \\
\addlinespace
\multicolumn{7}{l}{Full-duration components} \\
Flicker & 1.000 & 1.000 & 1.000 & 0.999 & 0.999 & 0.998 \\
Harmonics & 0.999 & 1.000$^{\ast}$ & 1.000 & 0.999 & 0.999 & 0.998 \\
Flicker+Harmonics & 0.999 & 1.000 & 1.000 & 0.999 & 0.998 & 0.998 \\
Flicker+Sag & 0.999 & 1.000 & 1.000 & 0.999 & 0.999 & 0.998 \\
Flicker+Swell & 1.000 & 1.000 & 1.000 & 1.000 & 0.999 & 0.998 \\
Interruption+Harmonics & 0.998 & 1.000 & 1.000 & 0.999 & 0.999 & 0.998 \\
Sag+Harmonics & 1.000 & 1.000 & 1.000 & 0.999 & 0.999 & 0.999 \\
Swell+Harmonics & 1.000 & 1.000 & 1.000 & 0.999 & 0.998 & 0.998 \\
\bottomrule
\end{tabular}
\end{table}

\begin{table}[!htbp]

\scriptsize\centering
\setlength{\tabcolsep}{4pt}
\renewcommand{\arraystretch}{1.12}
\caption{IoU by disturbance class for MAP and full-network diagonal LA under the primary selection rule ($c=y$, $S=100$). $^{\ast}$ marks 199 valid maps rather than 200. The remaining class entries use 200. Classes follow the same physical grouping in both tables.}
\label{tb:iou}
\begin{tabular}{lrrrrrr}
\toprule
Class & MAP & $q_{05}$ & $q_{25}$ & $q_{50}$ & $q_{75}$ & $q_{95}$ \\
\midrule
Class macro & 0.677 & 0.643 & 0.661 & 0.669 & 0.671 & 0.660 \\
\midrule
\multicolumn{7}{l}{Amplitude changes and oscillatory transient} \\
Sag & 0.299 & 0.276 & 0.296 & 0.287 & 0.279 & 0.256 \\
Swell & 0.748 & 0.603 & 0.706 & 0.720 & 0.712 & 0.675 \\
Interruption & 0.152 & 0.117 & 0.111 & 0.125 & 0.146 & 0.150 \\
Oscillatory transient & 0.703 & 0.606 & 0.681 & 0.682 & 0.688 & 0.615 \\
\addlinespace
\multicolumn{7}{l}{Short or periodically sparse events} \\
Impulsive transient & 0.235 & 0.013 & 0.099 & 0.201 & 0.219 & 0.202 \\
Notch & 0.035 & 0.025 & 0.014 & 0.029 & 0.025 & 0.016 \\
Spike & 0.008 & 0.012 & 0.015 & 0.013 & 0.009 & 0.007 \\
\addlinespace
\multicolumn{7}{l}{Full-duration components} \\
Flicker & 0.998 & 1.000 & 0.998 & 0.997 & 0.997 & 0.997 \\
Harmonics & 0.997 & 1.000$^{\ast}$ & 0.999 & 0.997 & 0.997 & 0.997 \\
Flicker+Harmonics & 0.997 & 1.000 & 0.997 & 0.997 & 0.997 & 0.997 \\
Flicker+Sag & 0.997 & 1.000 & 0.998 & 0.998 & 0.998 & 0.997 \\
Flicker+Swell & 0.999 & 1.000 & 0.999 & 0.999 & 0.998 & 0.997 \\
Interruption+Harmonics & 0.997 & 0.998 & 0.998 & 0.998 & 0.998 & 0.997 \\
Sag+Harmonics & 0.998 & 1.000 & 0.998 & 0.998 & 0.998 & 0.997 \\
Swell+Harmonics & 0.999 & 1.000 & 0.999 & 0.998 & 0.997 & 0.997 \\
\bottomrule
\end{tabular}
\end{table}

\clearpage
Per-class behavior differs across disturbance families. For Swell and Oscillatory transient, the $q_{05}$ RMA gains are approximately 0.190 and 0.113, respectively. RMA then decreases steadily from $q_{05}$ to $q_{95}$: from 0.627 to 0.363 for Swell and from 0.729 to 0.180 for Oscillatory transient. Interruption exhibits the reverse trend, increasing from 0.078 to 0.187 and exceeding MAP at $q_{75}$ and $q_{95}$. Sag shows no monotonic trend and remains below MAP at all five percentiles.

IoU shows no common monotonic trend across these localized disturbances. Among the five percentile summaries, IoU peaks at $q_{25}$ for Sag, $q_{50}$ for Swell, $q_{75}$ for Oscillatory transient and Impulsive transient, and $q_{95}$ for Interruption. Each peak remains below its corresponding MAP value. In particular, the $q_{05}$ RMA gains for Swell and Oscillatory transient accompany lower IoU. A greater share of attribution mass inside the disturbance can therefore coexist with fewer of the top-$L$ locations falling inside it.

For Impulsive transient, Notch and Spike, RMA remains low across percentile levels. Their short or periodically sparse support interacts with the 64-sample window and relevance averaging. Eight classes contain a full-duration component, producing nearly full reference masks and near-ceiling scores for both MAP and percentile explanations. The two localization scores therefore show little variation across percentile levels for these classes. Table~\ref{tb:guidance} summarizes the observed patterns by disturbance family.

\begin{table}[!htbp]

\scriptsize\centering
\setlength{\tabcolsep}{4pt}
\renewcommand{\arraystretch}{1.12}
\caption{Percentile-dependent localization patterns by disturbance family.}
\label{tb:guidance}
\resizebox{\linewidth}{!}{\begin{tabular}{>{\leavevmode}p{3.6cm}>{\leavevmode}p{8.8cm}}
\toprule
Family & Observed behavior and interpretation \\
\midrule
Swell, oscillatory transient & Lower-percentile summaries concentrate a larger fraction of attribution mass within the disturbance. The $q_{05}$ summary provides the highest RMA. The accompanying IoU reduction distinguishes mass concentration from temporal overlap. \\
Interruption & The in-region attribution fraction increases with percentile. The $q_{75}$ and $q_{95}$ summaries exceed MAP in RMA. \\
Sag & RMA varies non-monotonically across percentiles and remains below MAP. \\
Impulsive transient, notch, spike & RMA changes little across percentiles. These disturbances occupy short or periodically sparse intervals relative to the occlusion window. \\
Flicker, harmonics, full-duration composites & Both localization scores remain near 1 across percentiles, consistent with disturbance masks covering almost the entire waveform. \\
\bottomrule
\end{tabular}}
\end{table}

\subsection{Comparison with Alternative Bayesian Inference Schemes}
\label{sec:bnn_comparison}
Full-network diagonal LA is compared with last-layer diagonal, KFAC and full LA under both posterior-selection rules to examine the effects of the parameter subset and covariance structure. All three last-layer variants treat the same output-layer parameters as random, and each posterior variant follows the selection procedure of Section~\ref{sec:laplace}. This comparison isolates covariance structure within a fixed feature model and permits a full covariance representation at the last-layer dimension.

For trained MC-dropout \cite{gal2016dropout}, dropout at rate 0.2 is used during training and inference, and its own dropout-disabled network is the matched deterministic control. A separate post-hoc dropout perturbation is applied to the deployed model. Within each draw the dropout masks are fixed across the original and all occluded inputs. The deep ensemble \cite{lakshminarayanan2017simple} comprises five DCNNs trained under the training protocol of Section~\ref{sec:4.1} and differing only in random seed. Sampling budgets differ by construction: Laplace and MC-dropout draw $S=100$ samples, whereas the ensemble has five members, so its $\alpha=5$ and $\alpha=95$ summaries coincide with the member-wise extremes (Section~\ref{sec:sensitivity_exp} examines the sample-budget effect). All methods use the same localization operator and masks. Training requirements and online costs differ.
\begin{table}[tbp]

\scriptsize\centering
\setlength{\tabcolsep}{4pt}
\renewcommand{\arraystretch}{1.12}
\caption{Class-macro localization means ($c=y$). $+$/$-$ denote significant paired increases/decreases after Holm correction, relative to MAP except trained MC-dropout, which is paired with its own dropout-disabled control. Means may use different valid signals. Tests use matched pairs. LA uses the primary selection rule unless marked EB (empirical Bayes). All other summaries and $c=\hat y$ are retained in \ref{app:complete}.}
\label{tb:bnn}
\begin{tabular}{lrrrr}
\toprule
Method & Summary & Valid $n$ & RMA & IoU \\
\midrule
MAP & MAP & 3000 & 0.6408 & 0.6775 \\
All/diag & $q_{05}$ & 2999 & 0.6556$^{+}$ & 0.6433$^{-}$ \\
LL/diag & $q_{05}$ & 3000 & 0.6492$^{+}$ & 0.6602$^{-}$ \\
LL/KFAC & $q_{05}$ & 3000 & 0.6497$^{+}$ & 0.6548$^{-}$ \\
LL/full & $q_{05}$ & 3000 & 0.6484$^{+}$ & 0.6561$^{-}$ \\
All/diag (EB) & $q_{05}$ & 2957 & 0.6484$^{+}$ & 0.6271$^{-}$ \\
LL/diag (EB) & $q_{05}$ & 2999 & 0.6489$^{+}$ & 0.6593$^{-}$ \\
LL/KFAC (EB) & $q_{05}$ & 3000 & 0.6504$^{+}$ & 0.6652$^{-}$ \\
LL/full (EB) & $q_{05}$ & 2972 & 0.6230$^{-}$ & 0.6138$^{-}$ \\
Ensemble-5 & $q_{05}$ & 3000 & 0.6672$^{+}$ & 0.6709 \\
Post-hoc dropout & $q_{05}$ & 3000 & 0.6560$^{+}$ & 0.6648$^{-}$ \\
Dropout control & MAP & 3000 & 0.6704 & 0.6937 \\
Trained MC-dropout & $q_{05}$ & 2901 & 0.6834$^{+}$ & 0.6972 \\
\bottomrule
\end{tabular}
\end{table}
Table~\ref{tb:bnn} reports the comparison. The primary last-layer $q_{05}$ summaries increase RMA by 0.0076--0.0089 relative to deployed MAP, with significant IoU losses. The full-network $q_{05}$ RMA gain is larger, but so is its IoU loss. The ranking of diagonal, KFAC and full last-layer alternatives depends on the localization criterion. The empirical Bayes control changes the distribution and, for some variants, increases the frequency of zero summaries and localization losses.

The ensemble $q_{05}$ summary has a larger RMA increase than the primary LA variants and no significant IoU difference from deployed MAP. The dropout-disabled control for trained MC-dropout has an RMA of 0.6704, versus 0.6408 for the deployed model. Relative to its own control, its $q_{05}$ RMA increase is 0.0135, whereas its IoU difference is not significant. Thus the larger raw difference from deployed MAP cannot be assigned solely to uncertainty sampling.

\subsection{Comparison with Deterministic Attribution Baselines}
\label{sec:xai_comparison}
B-explanations are compared against deterministic occlusion, LIME \cite{ribeiro2016trust}, and KernelSHAP \cite{lundberg2017unified} using the localization metrics of Section~\ref{sec:loc_metrics} and the faithfulness metrics of Section~\ref{sec:faith_metrics}, on ten fixed instances per disturbance class from the test pool (150 in total). LIME and KernelSHAP use 128 perturbations, 16-sample feature groups, a zero baseline and the deployed classifier's target-class probabilities. LIME uses a linear surrogate with an $L_1$ penalty of strength 0.01. After the common normalization, 86 of its 150 maps are all zero. The reported LIME scores use the remaining 64 signals, and paired comparisons use MAP scores for those same signals. Table~\ref{tb:faith} summarizes the results.
\begin{table}[tbp]

\scriptsize\centering
\setlength{\tabcolsep}{4pt}
\renewcommand{\arraystretch}{1.12}
\caption{Localization ($c=y$) and faithfulness ($c=\hat y$, scored by deployed MAP) on the fixed subset. All LA rows use $S=100$. Del., Ins. and Mono. denote deletion AUC, insertion AUC and monotonicity. Lower deletion AUC and higher insertion AUC or monotonicity indicate greater faithfulness. LIME scores use 64 signals. Comparisons with MAP use the same signals.}
\label{tb:faith}
\begin{tabular}{lrrrrrrr}
\toprule
Method & Summary & $n$ & RMA & IoU & Del. & Ins. & Mono. \\
\midrule
MAP & MAP & 150 & 0.639 & 0.666 & 0.309 & 0.487 & 0.224 \\
All/diag & Mean & 150 & 0.614 & 0.659 & 0.322 & 0.484 & 0.208 \\
All/diag & $q_{05}$ & 149 & 0.658 & 0.645 & 0.327 & 0.486 & 0.129 \\
All/diag & $q_{95}$ & 150 & 0.609 & 0.651 & 0.343 & 0.481 & 0.160 \\
LL/diag & $q_{05}$ & 150 & 0.651 & 0.652 & 0.326 & 0.481 & 0.192 \\
LL/KFAC & $q_{05}$ & 150 & 0.656 & 0.657 & 0.328 & 0.481 & 0.201 \\
LL/full & $q_{05}$ & 150 & 0.653 & 0.660 & 0.334 & 0.478 & 0.186 \\
KernelSHAP-128 & MAP & 150 & 0.625 & 0.601 & 0.423 & 0.487 & 0.014 \\
LIME-128 & MAP & 64 & 0.816 & 0.741 & 0.469 & 0.769 & 0.092 \\
\bottomrule
\end{tabular}
\end{table}
The full-network $q_{05}$ paired RMA increase is 0.022 (unadjusted 95\% CI [0.005, 0.038]), but is not significant after Holm correction ($p=0.482$). The full-network $q_{05}$ summary has higher deletion AUC (paired difference +0.017, Holm $p=0.015$) and lower monotonicity ($-0.098$, Holm $p<0.001$) than deterministic occlusion. No insertion-AUC comparison is significant after correction.

For LIME's 64 evaluated signals, its RMA is 0.816 and the matched MAP RMA is 0.770. This conditional difference is not significant after correction. On these same signals, insertion AUC is 0.769 for LIME and 0.726 for MAP. KernelSHAP yields 150 valid maps and has lower IoU, higher deletion AUC and lower monotonicity at the evaluated budget. The sample-budget and runtime comparisons are reported in Sections~\ref{sec:sensitivity_exp} and \ref{sec:timing}.

\subsection{Sensitivity Analyses}
\label{sec:sensitivity_exp}
The number of posterior samples is varied over $S\in\{5,20,50,100\}$ on the same fixed 150-signal subset. With posterior parameters fixed, each budget uses the first $S$ of 100 posterior draws per signal. The mean, $q_{05}$, $q_{50}$ and $q_{95}$ are all reported. Table~\ref{tb:sens} gives the full-network results and \ref{app:budgets} includes all methods and both attribution targets.
\begin{table}[tbp]

\scriptsize\centering
\setlength{\tabcolsep}{4pt}
\renewcommand{\arraystretch}{1.12}
\caption{Full-network diagonal LA sampling sensitivity under the primary selection rule ($c=y$). Means use the valid maps at each budget.}
\label{tb:sens}
\begin{tabular}{lrrrr}
\toprule
$S$ & Summary & Valid $n$ & RMA & IoU \\
\midrule
5 & Mean & 150 & 0.6146 & 0.6457 \\
5 & $q_{05}$ & 148 & 0.6415 & 0.6247 \\
5 & $q_{50}$ & 150 & 0.6275 & 0.6524 \\
5 & $q_{95}$ & 150 & 0.6067 & 0.6343 \\
20 & Mean & 150 & 0.6136 & 0.6465 \\
20 & $q_{05}$ & 142 & 0.6444 & 0.6158 \\
20 & $q_{50}$ & 150 & 0.6250 & 0.6563 \\
20 & $q_{95}$ & 150 & 0.6076 & 0.6369 \\
50 & Mean & 150 & 0.6134 & 0.6545 \\
50 & $q_{05}$ & 149 & 0.6543 & 0.6404 \\
50 & $q_{50}$ & 150 & 0.6234 & 0.6580 \\
50 & $q_{95}$ & 150 & 0.6077 & 0.6432 \\
100 & Mean & 150 & 0.6143 & 0.6593 \\
100 & $q_{05}$ & 149 & 0.6579 & 0.6447 \\
100 & $q_{50}$ & 150 & 0.6226 & 0.6583 \\
100 & $q_{95}$ & 150 & 0.6085 & 0.6514 \\
\bottomrule
\end{tabular}
\end{table}
Across the 142 signals with valid scores at every budget, $q_{05}$ RMA increased from 0.633 at $S=5$ to 0.648 at $S=100$. Under the $\lceil\alpha S/100\rceil$ rule in Section~\ref{sec:uq}, $q_{05}$ is the sample minimum at $S=5$ and 20, and the third and fifth order statistics at $S=50$ and 100, respectively. The main evaluation uses $S=100$. Runtime sensitivity is reported separately in Section~\ref{sec:timing}. The $S=5$ setting matches the ensemble member count, while the measured times account for the different costs of ensemble predictions, full-network linearized draws and sampled last-layer heads.

\subsection{Explanation Behavior under Injected Measurement Noise}
\label{sec:noise_exp}
We injected additive white Gaussian noise into a fixed subset of 750 test signals at signal-to-noise ratios (SNRs) of 30, 20, and 10~dB relative to signal power. In additional conditions, we superimposed fifth and seventh harmonics with random phase and amplitudes of 5\% and 3\% of the fundamental amplitude, respectively. We repeated the percentile protocol on the matched signals using the original disturbance masks, which exclude the added harmonics. The analysis examines how localization and explanation dispersion respond to noise in a classifier trained on noise-free signals.

\begin{table}[tbp]

\scriptsize\centering
\setlength{\tabcolsep}{3.9pt}
\renewcommand{\arraystretch}{1.12}
\caption{Noise sensitivity on 750 matched signals. RMA/IoU entries are valid-instance means for $q_{05}$. Band is the mean pointwise $q_{95}-q_{05}$ width. Accuracy is that of the deployed MAP classifier. Both LA methods use the primary selection rule and $S=100$. The LL columns refer to LL/KFAC, and harm. denotes added background harmonics.}
\label{tb:noise}
\begin{tabular}{lrrrrrrr}
\toprule
Condition & Acc. & All RMA $y$ & All RMA $\hat y$ & All IoU $\hat y$ & All band & LL/KFAC RMA $\hat y$ & LL band \\
\midrule
Clean & 0.991 & 0.681 & 0.681 & 0.667 & 0.827 & 0.675 & 0.684 \\
30 dB & 0.436 & 0.615 & 0.630 & 0.621 & 0.876 & 0.635 & 0.802 \\
30 dB + harm. & 0.397 & 0.616 & 0.630 & 0.619 & 0.882 & 0.636 & 0.820 \\
20 dB & 0.237 & 0.605 & 0.618 & 0.606 & 0.897 & 0.630 & 0.876 \\
20 dB + harm. & 0.237 & 0.608 & 0.620 & 0.608 & 0.899 & 0.627 & 0.878 \\
10 dB & 0.153 & 0.615 & 0.619 & 0.602 & 0.933 & 0.625 & 0.900 \\
10 dB + harm. & 0.152 & 0.616 & 0.617 & 0.599 & 0.935 & 0.625 & 0.903 \\
\bottomrule
\end{tabular}
\end{table}
Table~\ref{tb:noise} reports the results as valid-instance averages. For this classifier trained on noise-free synthetic data, accuracy decreases from 0.991 on clean inputs to approximately 0.15 at 10~dB. The $c=y$ scores describe localization for the labeled disturbance, whereas the $c=\hat y$ scores describe attribution for the model's predicted class. In both cases, the localization mask describes the physical disturbance present in the signal. Faithfulness to the predicted decision is assessed separately in Section~\ref{sec:faith_metrics}. Mean explanation bands widen with noise in the evaluated LA variants.

Localization changes differ by disturbance group. On the matched subset of five localized disturbance classes, MAP IoU falls from 0.444 on clean inputs to 0.160 at 10~dB. Full-network $q_{05}$ falls from 0.338 to 0.142 and LL/KFAC $q_{05}$ from 0.363 to 0.136 ($c=y$). Near-ceiling scores for full-duration masks persist and can obscure this loss in overall averages.

\subsection{Field-Recording Case Study}
\label{sec:realworld}
\begin{figure*}[htbp]
	\captionsetup[subfigure]{labelformat=empty}
	\centering
	\begin{subfigure}[b]{0.4\textwidth}
		\centering
		\resizebox{\linewidth}{1.6cm}{%
			\includegraphics[scale=1,width=\textwidth]{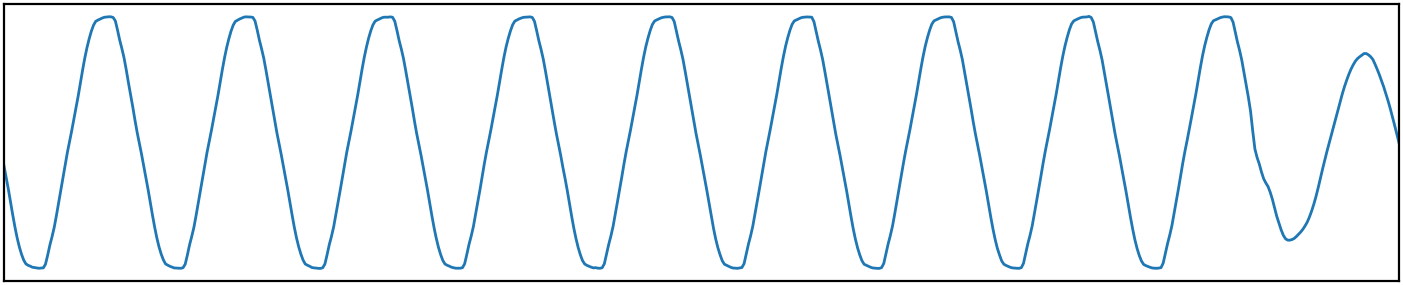}}
		\caption{Original Signal 1}
	\end{subfigure}
	\begin{subfigure}[b]{0.4\textwidth}
		\centering
		\resizebox{\linewidth}{1.6cm}{%
			\includegraphics[scale=1,width=\textwidth]{Realword_Origin.png}}
		\caption{Original Signal 1}
	\end{subfigure}

	\begin{subfigure}[b]{0.4\textwidth}
		\centering
		\resizebox{\linewidth}{0.9cm}{%
			\includegraphics[scale=1,width=\textwidth]{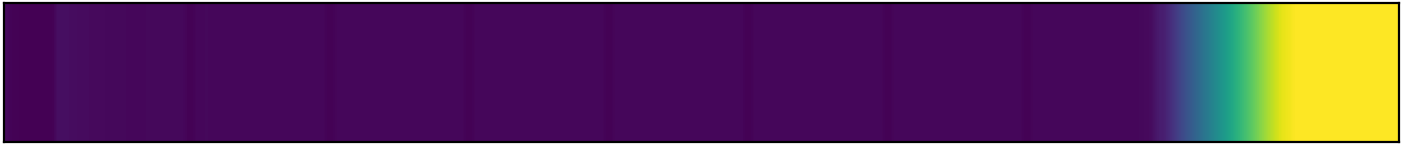}
		}
		\caption{MAP Explanation}
	\end{subfigure}
	\begin{subfigure}[b]{0.4\textwidth}
		\centering
		\resizebox{\linewidth}{0.9cm}{%
			\includegraphics[scale=1,width=\textwidth]{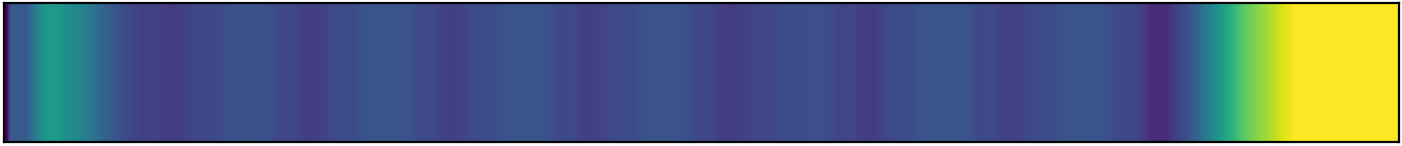}
		}
		\caption{MAP Explanation}
	\end{subfigure}

	\begin{subfigure}[b]{0.4\textwidth}
		\centering
		\resizebox{\linewidth}{0.9cm}{%
			\includegraphics[scale=1,width=\textwidth]{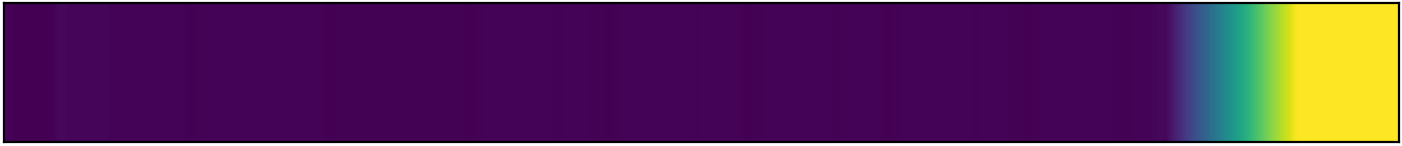}
		}
		\caption{B-Explanation $\alpha=5$}
	\end{subfigure}
	\begin{subfigure}[b]{0.4\textwidth}
		\centering
		\resizebox{\linewidth}{0.9cm}{%
			\includegraphics[scale=1,width=\textwidth]{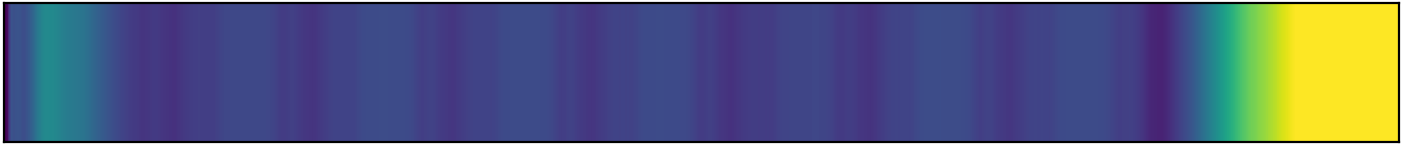}
		}
		\caption{B-Explanation $\alpha=5$}
	\end{subfigure}

	\begin{subfigure}[b]{0.4\textwidth}
		\centering
		\resizebox{\linewidth}{0.9cm}{%
			\includegraphics[scale=1,width=\textwidth]{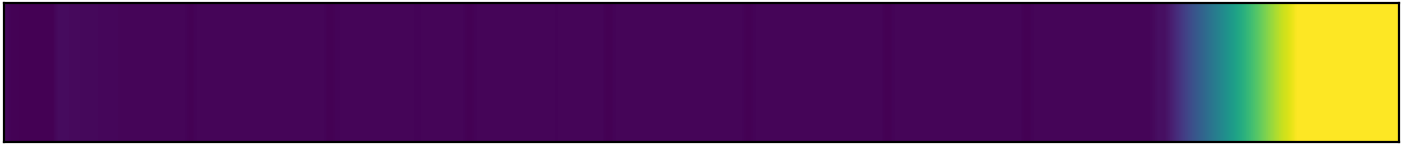}
		}
		\caption{B-Explanation $\alpha=25$}
	\end{subfigure}
	\begin{subfigure}[b]{0.4\textwidth}
		\centering
		\resizebox{\linewidth}{0.9cm}{%
			\includegraphics[scale=1,width=\textwidth]{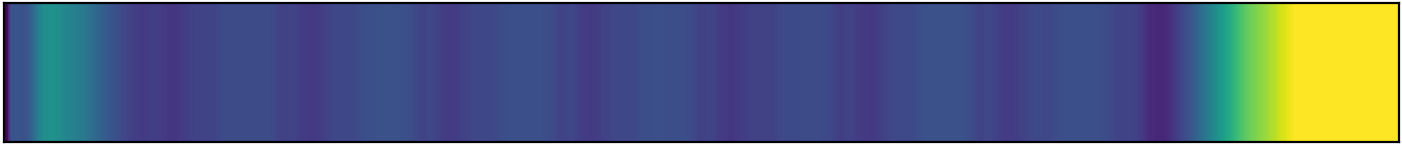}
		}
		\caption{B-Explanation $\alpha=25$}
	\end{subfigure}

	\begin{subfigure}[b]{0.4\textwidth}
		\centering
		\resizebox{\linewidth}{0.9cm}{%
			\includegraphics[scale=1,width=\textwidth]{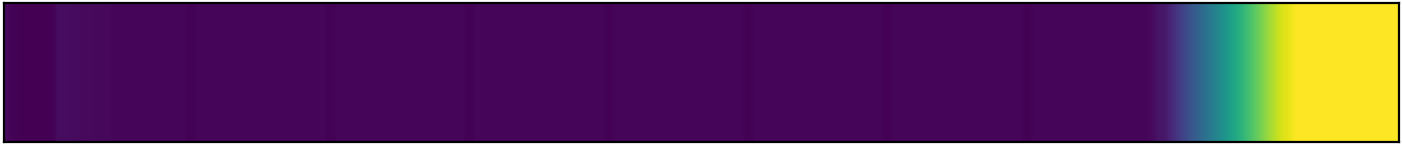}
		}
		\caption{B-Explanation $\alpha=50$}
	\end{subfigure}
	\begin{subfigure}[b]{0.4\textwidth}
		\centering
		\resizebox{\linewidth}{0.9cm}{%
			\includegraphics[scale=1,width=\textwidth]{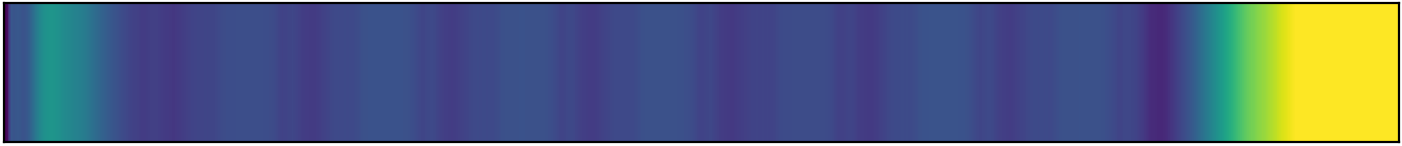}
		}
		\caption{B-Explanation $\alpha=50$}
	\end{subfigure}

	\begin{subfigure}[b]{0.4\textwidth}
		\centering
		\resizebox{\linewidth}{0.9cm}{%
			\includegraphics[scale=1,width=\textwidth]{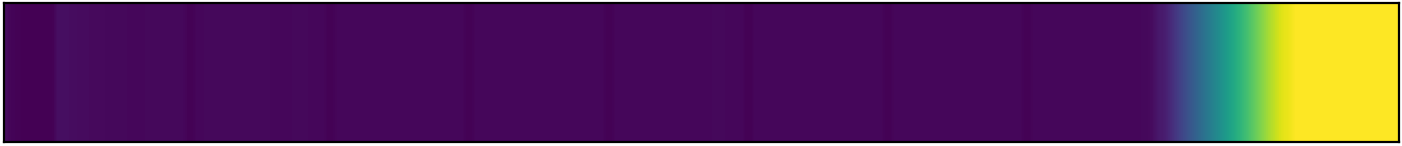}
		}
		\caption{B-Explanation $\alpha=75$}
	\end{subfigure}
	\begin{subfigure}[b]{0.4\textwidth}
		\centering
		\resizebox{\linewidth}{0.9cm}{%
			\includegraphics[scale=1,width=\textwidth]{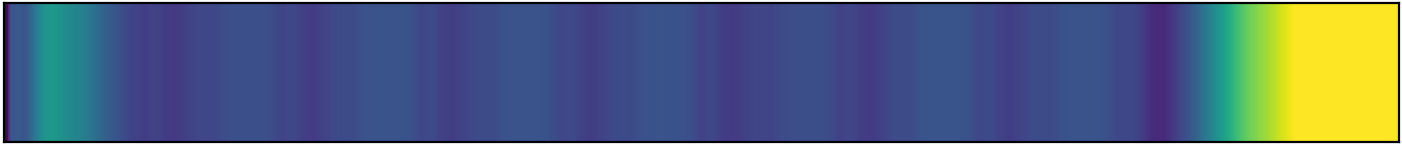}
		}
		\caption{B-Explanation $\alpha=75$}
	\end{subfigure}

	\begin{subfigure}[b]{0.4\textwidth}
		\centering
		\resizebox{\linewidth}{0.9cm}{%
			\includegraphics[scale=1,width=\textwidth]{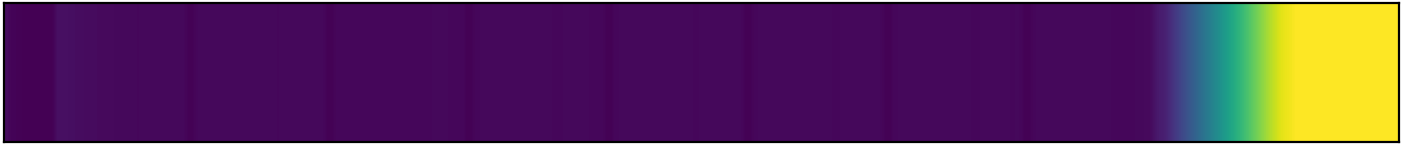}
		}
		\caption{\centering{B-Explanation $\alpha=95$ \\ (a) Model 1}}
	\end{subfigure}
	\begin{subfigure}[b]{0.4\textwidth}
		\centering
		\resizebox{\linewidth}{0.9cm}{%
			\includegraphics[scale=1,width=\textwidth]{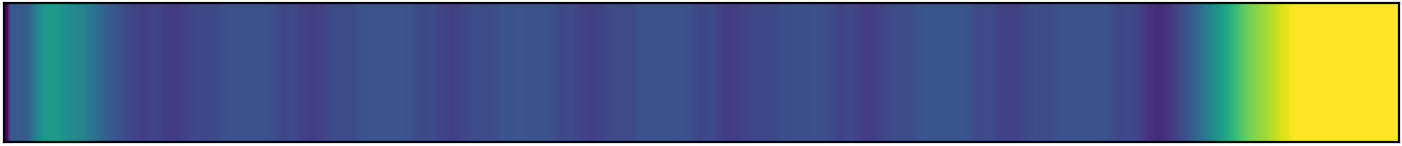}
		}
		\caption{\centering{B-Explanation $\alpha=95$ \\ (b) Model 2}}
	\end{subfigure}

	\begin{center}
		\begin{subfigure}[b]{0.4\textwidth}
			\centering
			\resizebox{\linewidth}{0.45cm}{%
				\includegraphics[scale=1,width=\textwidth]{colormap.png}
			}
		\end{subfigure}
	\end{center}
	\caption{MAP and Bayesian explanations for real-world Sag signal prediction (signal 1). Both MAP models, trained on the synthetic dataset with different random seeds, classify the signal correctly as ``Sag''.}
	\label{fig:realworld}
\end{figure*}

\begin{figure*}[htbp]
	\captionsetup[subfigure]{labelformat=empty}
	\centering
	\begin{subfigure}[b]{0.4\textwidth}
		\centering
		\resizebox{\linewidth}{1.6cm}{%
			\includegraphics[scale=1,width=\textwidth]{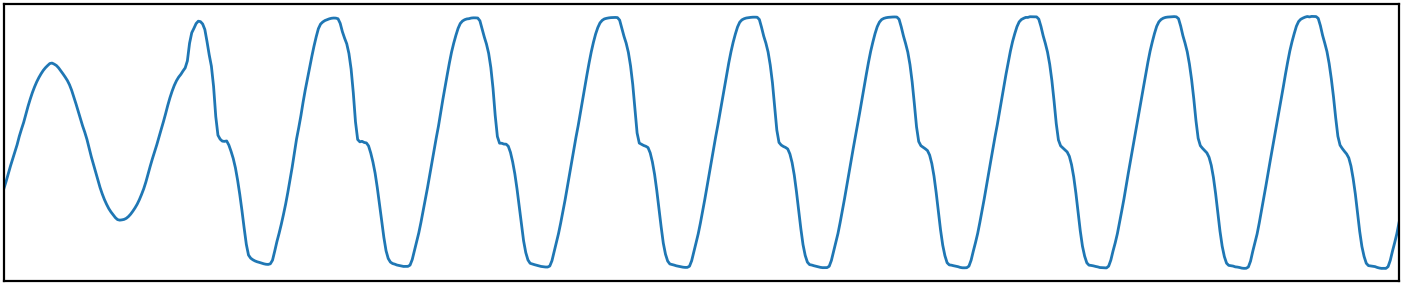}}
		\caption{Original Signal 2}
	\end{subfigure}
	\begin{subfigure}[b]{0.4\textwidth}
		\centering
		\resizebox{\linewidth}{1.6cm}{%
			\includegraphics[scale=1,width=\textwidth]{Realword3_Origin.png}}
		\caption{Original Signal 2}
	\end{subfigure}

	\begin{subfigure}[b]{0.4\textwidth}
		\centering
		\resizebox{\linewidth}{0.9cm}{%
			\includegraphics[scale=1,width=\textwidth]{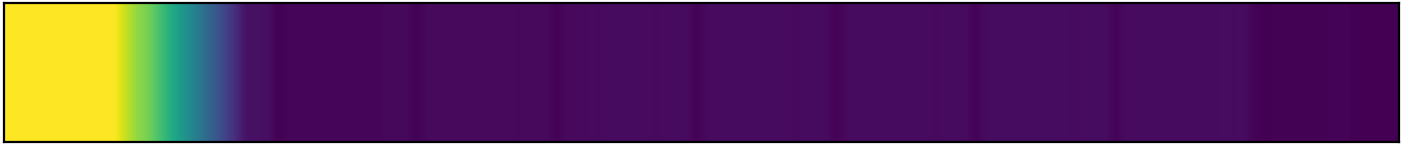}}
		\caption{MAP Explanation}
	\end{subfigure}
	\begin{subfigure}[b]{0.4\textwidth}
		\centering
		\resizebox{\linewidth}{0.9cm}{%
			\includegraphics[scale=1,width=\textwidth]{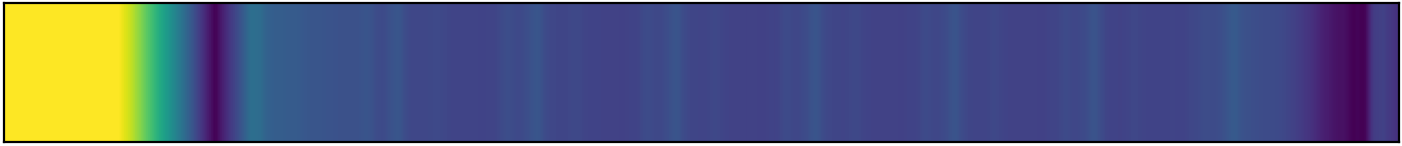}}
		\caption{MAP Explanation}
	\end{subfigure}

	\begin{subfigure}[b]{0.4\textwidth}
		\centering
		\resizebox{\linewidth}{0.9cm}{%
			\includegraphics[scale=1,width=\textwidth]{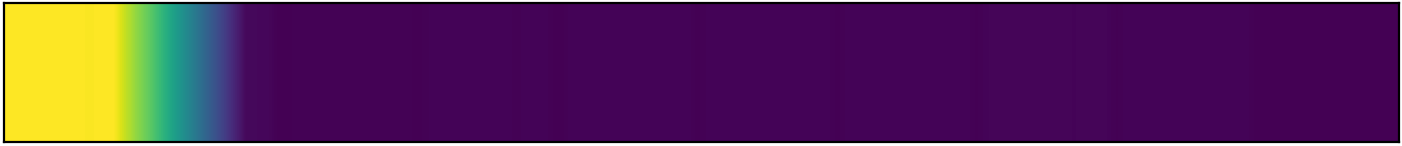}}
		\caption{B-Explanation $\alpha=5$}
	\end{subfigure}
	\begin{subfigure}[b]{0.4\textwidth}
		\centering
		\resizebox{\linewidth}{0.9cm}{%
			\includegraphics[scale=1,width=\textwidth]{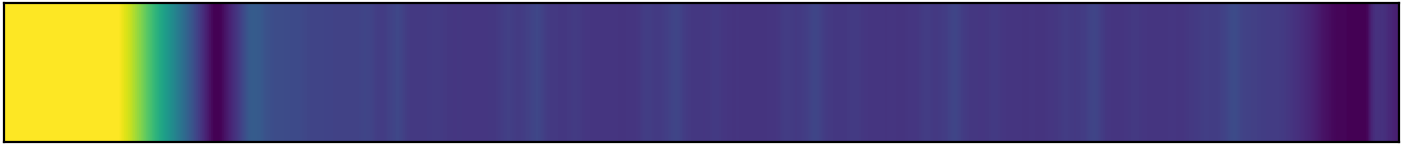}}
		\caption{B-Explanation $\alpha=5$}
	\end{subfigure}

	\begin{subfigure}[b]{0.4\textwidth}
		\centering
		\resizebox{\linewidth}{0.9cm}{%
			\includegraphics[scale=1,width=\textwidth]{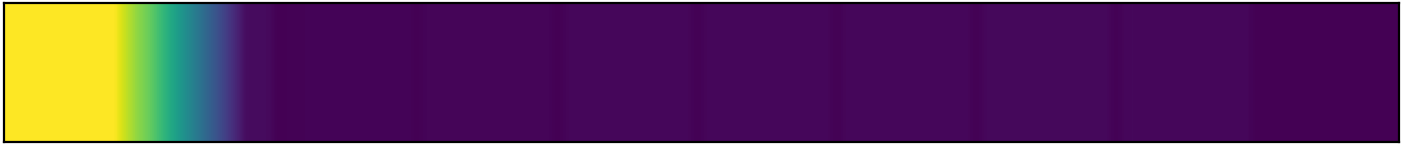}}
		\caption{B-Explanation $\alpha=25$}
	\end{subfigure}
	\begin{subfigure}[b]{0.4\textwidth}
		\centering
		\resizebox{\linewidth}{0.9cm}{%
			\includegraphics[scale=1,width=\textwidth]{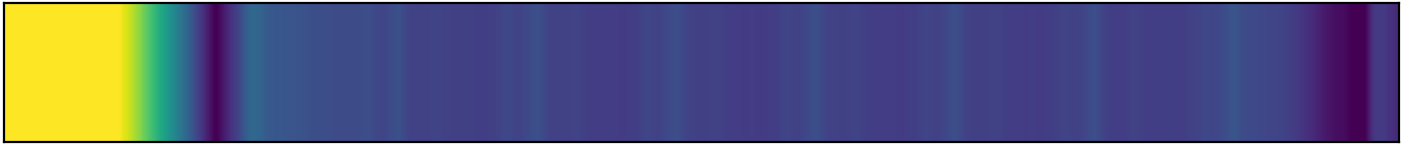}}
		\caption{B-Explanation $\alpha=25$}
	\end{subfigure}

	\begin{subfigure}[b]{0.4\textwidth}
		\centering
		\resizebox{\linewidth}{0.9cm}{%
			\includegraphics[scale=1,width=\textwidth]{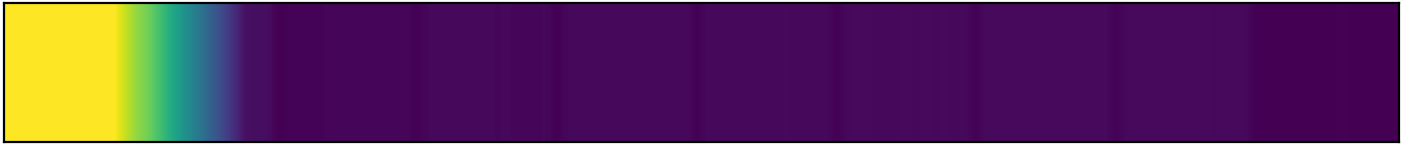}}
		\caption{B-Explanation $\alpha=50$}
	\end{subfigure}
	\begin{subfigure}[b]{0.4\textwidth}
		\centering
		\resizebox{\linewidth}{0.9cm}{%
			\includegraphics[scale=1,width=\textwidth]{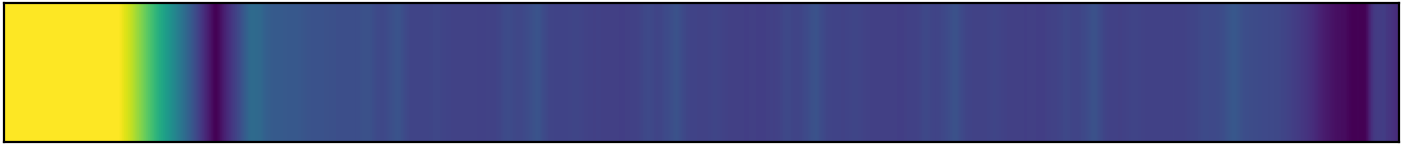}}
		\caption{B-Explanation $\alpha=50$}
	\end{subfigure}

	\begin{subfigure}[b]{0.4\textwidth}
		\centering
		\resizebox{\linewidth}{0.9cm}{%
			\includegraphics[scale=1,width=\textwidth]{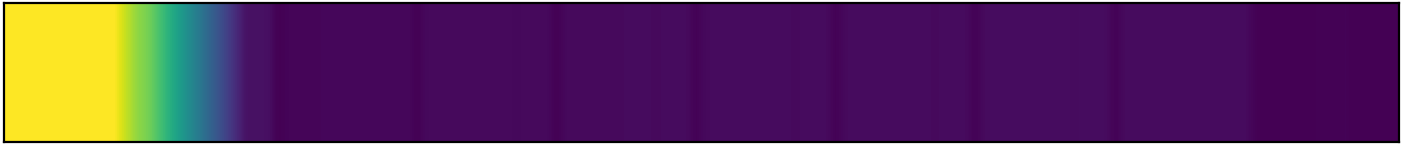}}
		\caption{B-Explanation $\alpha=75$}
	\end{subfigure}
	\begin{subfigure}[b]{0.4\textwidth}
		\centering
		\resizebox{\linewidth}{0.9cm}{%
			\includegraphics[scale=1,width=\textwidth]{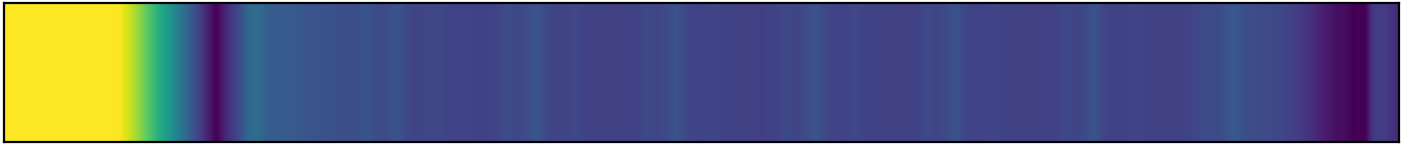}}
		\caption{B-Explanation $\alpha=75$}
	\end{subfigure}

	\begin{subfigure}[b]{0.4\textwidth}
		\centering
		\resizebox{\linewidth}{0.9cm}{%
			\includegraphics[scale=1,width=\textwidth]{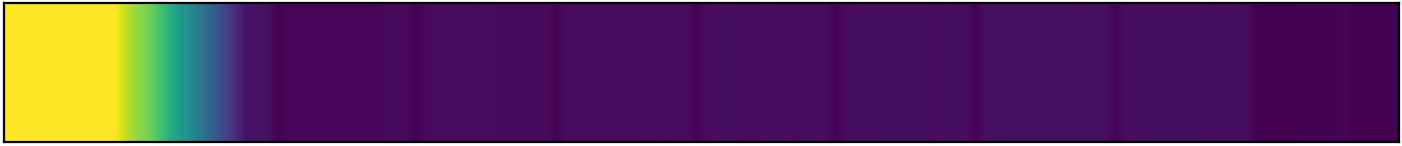}}
		\caption{\centering{B-Explanation $\alpha=95$ \\ (a) Model 1}}
	\end{subfigure}
	\begin{subfigure}[b]{0.4\textwidth}
		\centering
		\resizebox{\linewidth}{0.9cm}{%
			\includegraphics[scale=1,width=\textwidth]{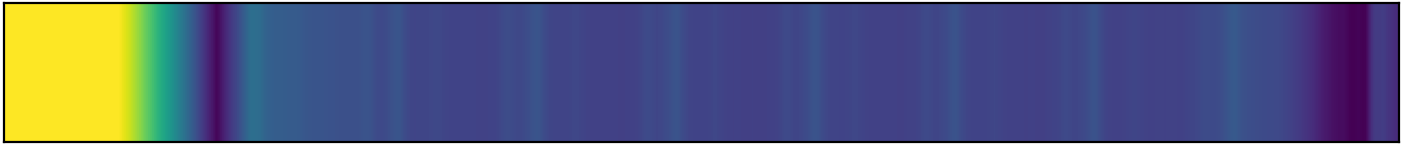}}
		\caption{\centering{B-Explanation $\alpha=95$ \\ (b) Model 2}}
	\end{subfigure}

	\begin{center}
		\begin{subfigure}[b]{0.4\textwidth}
			\centering
			\resizebox{\linewidth}{0.45cm}{%
				\includegraphics[scale=1,width=\textwidth]{colormap.png}
			}
		\end{subfigure}
	\end{center}
	\caption{MAP and Bayesian explanations for real-world Sag signal prediction (signal 2). Both MAP models, trained on the synthetic dataset with different random seeds, classify the signal correctly as ``Sag''.}
	\label{fig:realworld3}
\end{figure*}

To illustrate behavior under synthetic-to-field transfer, the method is applied to real recordings from the University of Cadiz power network \cite{h2k88d-17} (50~Hz, sampled at 20~kHz, five years of measurements with diverse sag events). Signals are filtered to prevent aliasing, downsampled to the training rate, windowed to 10 cycles containing sag events, and normalized, then passed to two DCNNs trained on the synthetic dataset with different random seeds. Representative results are shown in Figs.~\ref{fig:realworld} and \ref{fig:realworld3}.

Both models classify the two illustrated sag signals correctly after training on synthetic data. In all four model--signal combinations, the dominant relevance lies in the sag interval: the terminal sag of Signal 1 and the initial sag of Signal 2. Model 2 assigns visibly more residual relevance to nominal cycles than Model 1. Within each model, the displayed percentile maps change little from $\alpha=5$ to $\alpha=95$. Field localization is assessed qualitatively because disturbance-region annotations are unavailable. The broader field analysis below uses all 104 extracted windows as a distribution-shift case study.

\subsection{Out-of-Distribution Case Study}
\label{sec:ood}
The coupling of predictive and explanation samples (Section~\ref{sec:expl_dist}) suggests using explanation dispersion as an OOD indicator. Two OOD regimes are evaluated against the corresponding in-distribution references: (i) inputs with an off-nominal fundamental frequency (47 and 52~Hz, with 750 matched signals each), representing operating conditions outside the training distribution, and (ii) 104 ten-cycle windows from 26 recordings, obtained by segmenting the field recordings of Section~\ref{sec:realworld} into consecutive windows without manual selection (up to four per recording), representing the sensor and network conditions of field data. Explanation dispersion is computed for the fixed predicted class of each unoccluded input. For each input, predictive entropy, MI, and two explanation-dispersion statistics (mean $5$--$95$ band width and mean explanation standard deviation) are recorded, and each statistic is scored by AUROC for separating the OOD group from the in-distribution reference.

Synthetic shifts are compared with their matched clean signals, and field windows with the 3,000-signal clean reference. Larger scores define the OOD direction. Confidence intervals use the dependence-aware bootstrap of Section~\ref{sec:stats}.

\begin{table}[tbp]

\scriptsize\centering
\setlength{\tabcolsep}{4pt}
\renewcommand{\arraystretch}{1.12}
\caption{OOD AUROC and 95\% bootstrap intervals. Synthetic shifts use 750 matched clean and shifted signals. The field group contains 104 windows clustered by 26 recordings and uses 3,000 in-distribution reference signals. All LA rows use the primary selection rule. Band and Expl. std denote mean pointwise band width and explanation standard deviation. Constant deterministic dispersion statistics are omitted.}
\label{tb:ood}
\begin{tabular}{lrrrrr}
\toprule
Group & Method & Entropy & MI & Band & Expl. std \\
\midrule
47 Hz & MAP & \shortstack{0.722\\ {[0.699, 0.745]}} & -- & -- & -- \\
47 Hz & All/diag & \shortstack{0.621\\ {[0.603, 0.639]}} & \shortstack{0.611\\ {[0.594, 0.627]}} & \shortstack{0.521\\ {[0.496, 0.546]}} & \shortstack{0.508\\ {[0.484, 0.533]}} \\
47 Hz & LL/diag & \shortstack{0.839\\ {[0.819, 0.858]}} & \shortstack{0.841\\ {[0.822, 0.860]}} & \shortstack{0.556\\ {[0.534, 0.577]}} & \shortstack{0.548\\ {[0.528, 0.568]}} \\
47 Hz & LL/KFAC & \shortstack{0.815\\ {[0.795, 0.836]}} & \shortstack{0.824\\ {[0.804, 0.843]}} & \shortstack{0.567\\ {[0.548, 0.588]}} & \shortstack{0.560\\ {[0.541, 0.582]}} \\
47 Hz & LL/full & \shortstack{0.773\\ {[0.752, 0.795]}} & \shortstack{0.784\\ {[0.764, 0.804]}} & \shortstack{0.553\\ {[0.532, 0.574]}} & \shortstack{0.535\\ {[0.514, 0.556]}} \\
47 Hz & Ensemble-5 & \shortstack{0.713\\ {[0.691, 0.735]}} & \shortstack{0.737\\ {[0.715, 0.760]}} & \shortstack{0.581\\ {[0.557, 0.605]}} & \shortstack{0.584\\ {[0.560, 0.606]}} \\
52 Hz & MAP & \shortstack{0.784\\ {[0.760, 0.807]}} & -- & -- & -- \\
52 Hz & All/diag & \shortstack{0.740\\ {[0.717, 0.763]}} & \shortstack{0.722\\ {[0.699, 0.744]}} & \shortstack{0.343\\ {[0.322, 0.365]}} & \shortstack{0.409\\ {[0.387, 0.432]}} \\
52 Hz & LL/diag & \shortstack{0.883\\ {[0.867, 0.899]}} & \shortstack{0.882\\ {[0.867, 0.897]}} & \shortstack{0.549\\ {[0.525, 0.575]}} & \shortstack{0.573\\ {[0.547, 0.597]}} \\
52 Hz & LL/KFAC & \shortstack{0.865\\ {[0.847, 0.882]}} & \shortstack{0.873\\ {[0.857, 0.889]}} & \shortstack{0.546\\ {[0.522, 0.569]}} & \shortstack{0.574\\ {[0.550, 0.597]}} \\
52 Hz & LL/full & \shortstack{0.830\\ {[0.810, 0.851]}} & \shortstack{0.829\\ {[0.808, 0.850]}} & \shortstack{0.521\\ {[0.497, 0.546]}} & \shortstack{0.536\\ {[0.512, 0.561]}} \\
52 Hz & Ensemble-5 & \shortstack{0.780\\ {[0.753, 0.804]}} & \shortstack{0.798\\ {[0.774, 0.823]}} & \shortstack{0.411\\ {[0.388, 0.435]}} & \shortstack{0.419\\ {[0.396, 0.441]}} \\
Field & MAP & \shortstack{0.989\\ {[0.981, 0.996]}} & -- & -- & -- \\
Field & All/diag & \shortstack{0.939\\ {[0.905, 0.966]}} & \shortstack{0.840\\ {[0.795, 0.882]}} & \shortstack{0.539\\ {[0.487, 0.589]}} & \shortstack{0.511\\ {[0.454, 0.571]}} \\
Field & LL/diag & \shortstack{0.992\\ {[0.983, 0.998]}} & \shortstack{0.989\\ {[0.977, 0.998]}} & \shortstack{0.579\\ {[0.536, 0.626]}} & \shortstack{0.560\\ {[0.518, 0.603]}} \\
Field & LL/KFAC & \shortstack{0.995\\ {[0.993, 0.997]}} & \shortstack{0.995\\ {[0.992, 0.997]}} & \shortstack{0.690\\ {[0.638, 0.743]}} & \shortstack{0.643\\ {[0.599, 0.685]}} \\
Field & LL/full & \shortstack{0.994\\ {[0.991, 0.997]}} & \shortstack{0.993\\ {[0.989, 0.996]}} & \shortstack{0.693\\ {[0.644, 0.738]}} & \shortstack{0.637\\ {[0.593, 0.679]}} \\
Field & Ensemble-5 & \shortstack{0.996\\ {[0.994, 0.998]}} & \shortstack{0.996\\ {[0.995, 0.998]}} & \shortstack{0.728\\ {[0.676, 0.779]}} & \shortstack{0.727\\ {[0.672, 0.782]}} \\
\bottomrule
\end{tabular}
\end{table}
Table~\ref{tb:ood} shows that predictive uncertainty separates the field group more effectively than explanation dispersion in these experiments. The same ordering generally holds for the frequency shifts, with method-dependent strength.

\subsection{Computational Cost}
\label{sec:timing}
We measured explanation time on eight cores of an Intel Xeon Gold 6240R CPU. For each of five Flicker test signals, we took the median of three timed runs after one warm-up, then reported the median across signals. All timings exclude model loading and one-time posterior setup. LA and sampling-baseline timings include occlusion, prediction, normalization, and summary generation. LIME and KernelSHAP timings cover attribution generation only. Table~\ref{tb:timing} reports the results.
\begin{table}[tbp]

\scriptsize\centering
\setlength{\tabcolsep}{4pt}
\renewcommand{\arraystretch}{1.12}
\caption{Online seconds per explanation at the indicated sample budgets. Fixed-budget comparators: MAP 0.125 s, ensemble-5 0.585 s, LIME-128 0.226 s and KernelSHAP-128 0.239 s.}
\label{tb:timing}
\begin{tabular}{lrrrr}
\toprule
Method & $S=5$ & $S=20$ & $S=50$ & $S=100$ \\
\midrule
All/diag & 1.538 & 5.852 & 13.953 & 27.699 \\
LL/diag & 0.106 & 0.106 & 0.106 & 0.107 \\
LL/KFAC & 0.106 & 0.107 & 0.106 & 0.109 \\
LL/full & 0.107 & 0.108 & 0.109 & 0.111 \\
Post-hoc dropout & 0.603 & 2.073 & 5.351 & 10.455 \\
Trained MC-dropout & 0.778 & 2.704 & 6.886 & 13.532 \\
\bottomrule
\end{tabular}
\end{table}
At $S=100$, generating one full-network explanation takes 27.70~s, approximately 116--123 times the measured attribution times for LIME and KernelSHAP. Last-layer LA costs approximately 0.107--0.111~s and is comparable with MAP occlusion because feature extraction is reused across heads.

One-time curvature fitting on 11,800 training signals took approximately 3,765~s for full-network diagonal LA, 25~s for last-layer diagonal LA, 36~s for last-layer KFAC and 384~s for last-layer full LA. These times exclude selection of $\lambda$. Curvature accumulation requires access to the training data and the trained model.

\section{Discussion}
\label{sec:discussion}

\subsection{Physical Interpretation of Explanation Uncertainty}
\label{sec:physical}

The local approximate posterior represents variation among sampled decision functions, and the induced dispersion at time $n$ measures the extent to which those functions disagree about the evidential role of sample $n$. This dispersion measures attribution variation, whereas explanation error requires a specified reference attribution. Disturbance structure and the occlusion baseline influence the attribution pattern. For the full-network $q_{05}$ summary, Swell and Oscillatory transient have higher RMA than MAP but lower IoU. At the same percentile, Sag and Interruption have lower RMA than MAP. Zero-baseline occlusion may alter a near-zero interruption segment less than a nominal segment, offering a plausible explanation for weak in-event relevance.

For flicker and harmonic components that persist throughout a recording, relevance distributed across several cycles is consistent with the duration of the disturbance. Their near-full masks yield high localization scores for different attribution patterns. For composite disturbances, explanation patterns can emphasize different constituent events (Fig.~\ref{fig:example}), so the distribution indicates which component of, e.g., a Flicker+Sag signal drives an individual classification, information that a single map does not provide. The clustering example (Fig.~\ref{fig:cluster}) similarly displays joint attribution structure that pointwise summaries do not preserve.

At the onset and recovery of sags and swells, overlapping occlusion windows contain both disturbed and nominal samples, so the resulting relevance reflects their combined contribution to the classifier response. For sparse transients, notch and spike, the 64-sample window and overlap averaging can blur support finer than the operator scale.

The percentile coverage statement in Section~\ref{sec:uq} concerns draws from the chosen approximate explanation distribution. Empirical calibration against reference attributions is a separate assessment. The physical disturbance mask specifies event support, whereas occlusion relevance measures a classifier probability change. A coverage study therefore requires reference attributions defined for the same explanation target and on the same scale.

\pagebreak
\subsection{Deployment Considerations}
\label{sec:deployment}

In renewable-rich microgrids, changes in operating conditions affect decisions ranging from dispatch to voltage and frequency regulation \cite{mbasso2026closedloop}. For a monitor reporting a disturbance, an explanation distribution could show which waveform regions support the label and how consistently the sampled models identify those regions.

The Laplace fit and posterior selection are one-off offline steps performed after training. Deployment requires the trained MAP model and the posterior representation, whose storage depends on the parameter subset and covariance structure. The cost of B-explanation generation is characterized in Section~\ref{sec:timing}. One deployment option is to compute B-explanations at edge devices when an event is detected and transmit the percentile maps and band widths with the event record. Because the method is post hoc, it can be retrofitted to deployed classifiers without modifying the classification function. For multi-channel installations with three-phase voltage and current, per-channel or joint segment occlusion would require specifying the operator and evaluating cross-channel dependence. A joint treatment of cross-channel explanation correlation is left to future work (Section~\ref{sec:4}).

\section{Conclusion}
\label{sec:4}

This paper developed a post-hoc method that equips deployed PQD classifiers with uncertainty-aware explanations. A Laplace approximation fitted to an existing model and propagated through occlusion sensitivity yields an explanation distribution without retraining. Full-network inference uses linearized prediction, while last-layer inference uses a fixed feature extractor. Percentile B-explanations showed class-dependent behavior, with the primary lower-tail summary increasing RMA for swells and oscillatory transients but reducing IoU. Explanation dispersion increased under injected noise. The qualitative examples illustrate attribution patterns that a single map and aggregate localization scores do not distinguish.

\pagebreak
\begingroup\interlinepenalty=10000\relax
The synthetic experiments quantified localization against known disturbance intervals. In two real-world voltage sag examples, qualitative inspection showed that both classifiers assigned the greatest relevance to the sag interval in each waveform. Across the 104 real-world windows, predictive uncertainty distinguished the measured signals from synthetic test signals more effectively than explanation dispersion. Last-layer inference reduces online computation relative to full-network inference in the measured setting.\par\endgroup

The explanation intervals characterize posterior attribution variability, and future work should assess their empirical coverage in controlled settings with defined reference attributions. Annotated field measurements covering additional disturbance categories will support broader localization validation. Disturbance-specific percentile-selection rules will require evaluation on independent data. Comparisons under equal elapsed-time budgets and alternative perturbation baselines will further assess computational efficiency and faithfulness.

Further methodological directions include incremental B-explanations over sliding windows to reuse feature and occlusion computations, joint occlusion of three-phase voltage and current for cross-channel attribution, and low-rank or richer full-network curvature approximations with measured computational costs. Future deployment studies should replay recorded waveforms on the target monitoring device and measure explanation latency and peak memory use.

\clearpage

\par\bigskip
\appendix
\section{Complete Localization Summaries}\label{app:complete}
Table~\ref{tb:complete} retains all six summaries, both posterior-selection rules and both target definitions. Means are class-macro averages over valid maps. Valid counts are reported separately for each target, and paired effects are reported in \ref{app:paired}.
\begingroup\scriptsize
\setlength{\tabcolsep}{3pt}
\renewcommand{\arraystretch}{1.06}
\begin{longtable}{llrrrrrr}
\caption{Complete test-pool localization results. LA uses the primary selection rule unless marked EB (empirical Bayes). $n$ is the valid count out of 3,000.}\label{tb:complete}\\
\toprule
Method & Summary & $n_y$ & RMA $y$ & IoU $y$ & $n_{\hat y}$ & RMA $\hat y$ & IoU $\hat y$ \\
\midrule\endfirsthead
\toprule
Method & Summary & $n_y$ & RMA $y$ & IoU $y$ & $n_{\hat y}$ & RMA $\hat y$ & IoU $\hat y$ \\
\midrule\endhead
\midrule\multicolumn{8}{r}{Continued on next page}\\\endfoot
\bottomrule\endlastfoot
MAP & MAP & 3000 & 0.6408 & 0.6775 & 3000 & 0.6408 & 0.6771 \\
All/diag & Mean & 3000 & 0.6133 & 0.6639 & 3000 & 0.6133 & 0.6638 \\
All/diag & $q_{05}$ & 2999 & 0.6556 & 0.6433 & 2999 & 0.6556 & 0.6433 \\
All/diag & $q_{25}$ & 3000 & 0.6435 & 0.6605 & 3000 & 0.6435 & 0.6604 \\
All/diag & $q_{50}$ & 3000 & 0.6240 & 0.6693 & 3000 & 0.6240 & 0.6692 \\
All/diag & $q_{75}$ & 3000 & 0.6092 & 0.6705 & 3000 & 0.6092 & 0.6702 \\
All/diag & $q_{95}$ & 3000 & 0.6073 & 0.6598 & 3000 & 0.6073 & 0.6596 \\
LL/diag & Mean & 3000 & 0.6095 & 0.6660 & 3000 & 0.6095 & 0.6656 \\
LL/diag & $q_{05}$ & 3000 & 0.6492 & 0.6602 & 3000 & 0.6492 & 0.6602 \\
LL/diag & $q_{25}$ & 3000 & 0.6353 & 0.6666 & 3000 & 0.6353 & 0.6665 \\
LL/diag & $q_{50}$ & 3000 & 0.6190 & 0.6691 & 3000 & 0.6190 & 0.6688 \\
LL/diag & $q_{75}$ & 3000 & 0.6056 & 0.6697 & 3000 & 0.6056 & 0.6694 \\
LL/diag & $q_{95}$ & 3000 & 0.6034 & 0.6621 & 3000 & 0.6034 & 0.6618 \\
LL/KFAC & Mean & 3000 & 0.6092 & 0.6640 & 3000 & 0.6093 & 0.6639 \\
LL/KFAC & $q_{05}$ & 3000 & 0.6497 & 0.6548 & 3000 & 0.6497 & 0.6547 \\
LL/KFAC & $q_{25}$ & 3000 & 0.6335 & 0.6663 & 3000 & 0.6336 & 0.6663 \\
LL/KFAC & $q_{50}$ & 3000 & 0.6179 & 0.6709 & 3000 & 0.6179 & 0.6708 \\
LL/KFAC & $q_{75}$ & 3000 & 0.6056 & 0.6706 & 3000 & 0.6056 & 0.6702 \\
LL/KFAC & $q_{95}$ & 3000 & 0.6042 & 0.6614 & 3000 & 0.6042 & 0.6614 \\
LL/full & Mean & 3000 & 0.6104 & 0.6658 & 3000 & 0.6104 & 0.6657 \\
LL/full & $q_{05}$ & 3000 & 0.6484 & 0.6561 & 3000 & 0.6483 & 0.6561 \\
LL/full & $q_{25}$ & 3000 & 0.6339 & 0.6666 & 3000 & 0.6339 & 0.6666 \\
LL/full & $q_{50}$ & 3000 & 0.6183 & 0.6716 & 3000 & 0.6183 & 0.6713 \\
LL/full & $q_{75}$ & 3000 & 0.6068 & 0.6685 & 3000 & 0.6068 & 0.6682 \\
LL/full & $q_{95}$ & 3000 & 0.6047 & 0.6603 & 3000 & 0.6047 & 0.6603 \\
All/diag (EB) & Mean & 3000 & 0.6081 & 0.6599 & 3000 & 0.6081 & 0.6598 \\
All/diag (EB) & $q_{05}$ & 2957 & 0.6484 & 0.6271 & 2957 & 0.6483 & 0.6271 \\
All/diag (EB) & $q_{25}$ & 3000 & 0.6347 & 0.6495 & 3000 & 0.6347 & 0.6495 \\
All/diag (EB) & $q_{50}$ & 3000 & 0.6115 & 0.6623 & 3000 & 0.6115 & 0.6624 \\
All/diag (EB) & $q_{75}$ & 3000 & 0.6046 & 0.6566 & 3000 & 0.6046 & 0.6565 \\
All/diag (EB) & $q_{95}$ & 3000 & 0.6072 & 0.6457 & 3000 & 0.6072 & 0.6454 \\
LL/diag (EB) & Mean & 3000 & 0.6094 & 0.6668 & 3000 & 0.6094 & 0.6669 \\
LL/diag (EB) & $q_{05}$ & 2999 & 0.6489 & 0.6593 & 2999 & 0.6489 & 0.6595 \\
LL/diag (EB) & $q_{25}$ & 3000 & 0.6355 & 0.6670 & 3000 & 0.6355 & 0.6671 \\
LL/diag (EB) & $q_{50}$ & 3000 & 0.6188 & 0.6702 & 3000 & 0.6188 & 0.6701 \\
LL/diag (EB) & $q_{75}$ & 3000 & 0.6057 & 0.6695 & 3000 & 0.6057 & 0.6692 \\
LL/diag (EB) & $q_{95}$ & 3000 & 0.6034 & 0.6605 & 3000 & 0.6034 & 0.6605 \\
LL/KFAC (EB) & Mean & 3000 & 0.6184 & 0.6704 & 3000 & 0.6184 & 0.6701 \\
LL/KFAC (EB) & $q_{05}$ & 3000 & 0.6504 & 0.6652 & 3000 & 0.6505 & 0.6653 \\
LL/KFAC (EB) & $q_{25}$ & 3000 & 0.6381 & 0.6701 & 3000 & 0.6381 & 0.6698 \\
LL/KFAC (EB) & $q_{50}$ & 3000 & 0.6280 & 0.6727 & 3000 & 0.6280 & 0.6724 \\
LL/KFAC (EB) & $q_{75}$ & 3000 & 0.6169 & 0.6739 & 3000 & 0.6169 & 0.6736 \\
LL/KFAC (EB) & $q_{95}$ & 3000 & 0.6074 & 0.6717 & 3000 & 0.6074 & 0.6715 \\
LL/full (EB) & Mean & 3000 & 0.6030 & 0.6529 & 3000 & 0.6030 & 0.6530 \\
LL/full (EB) & $q_{05}$ & 2972 & 0.6230 & 0.6138 & 2969 & 0.6230 & 0.6138 \\
LL/full (EB) & $q_{25}$ & 3000 & 0.6163 & 0.6502 & 3000 & 0.6163 & 0.6501 \\
LL/full (EB) & $q_{50}$ & 3000 & 0.6015 & 0.6589 & 3000 & 0.6015 & 0.6587 \\
LL/full (EB) & $q_{75}$ & 3000 & 0.6010 & 0.6564 & 3000 & 0.6010 & 0.6559 \\
LL/full (EB) & $q_{95}$ & 3000 & 0.6047 & 0.6449 & 3000 & 0.6047 & 0.6445 \\
Ensemble-5 & Mean & 3000 & 0.6394 & 0.6758 & 3000 & 0.6394 & 0.6758 \\
Ensemble-5 & $q_{05}$ & 3000 & 0.6672 & 0.6709 & 3000 & 0.6672 & 0.6709 \\
Ensemble-5 & $q_{25}$ & 3000 & 0.6523 & 0.6707 & 3000 & 0.6523 & 0.6707 \\
Ensemble-5 & $q_{50}$ & 3000 & 0.6454 & 0.6722 & 3000 & 0.6455 & 0.6719 \\
Ensemble-5 & $q_{75}$ & 3000 & 0.6399 & 0.6743 & 3000 & 0.6399 & 0.6739 \\
Ensemble-5 & $q_{95}$ & 3000 & 0.6317 & 0.6756 & 3000 & 0.6317 & 0.6753 \\
Post-hoc dropout & Mean & 3000 & 0.6198 & 0.6701 & 3000 & 0.6198 & 0.6697 \\
Post-hoc dropout & $q_{05}$ & 3000 & 0.6560 & 0.6648 & 3000 & 0.6560 & 0.6646 \\
Post-hoc dropout & $q_{25}$ & 3000 & 0.6452 & 0.6706 & 3000 & 0.6452 & 0.6705 \\
Post-hoc dropout & $q_{50}$ & 3000 & 0.6325 & 0.6738 & 3000 & 0.6325 & 0.6734 \\
Post-hoc dropout & $q_{75}$ & 3000 & 0.6180 & 0.6756 & 3000 & 0.6180 & 0.6753 \\
Post-hoc dropout & $q_{95}$ & 3000 & 0.6067 & 0.6686 & 3000 & 0.6067 & 0.6683 \\
Dropout control & MAP & 3000 & 0.6704 & 0.6937 & 3000 & 0.6704 & 0.6936 \\
Trained MC-dropout & Mean & 3000 & 0.6629 & 0.6926 & 3000 & 0.6629 & 0.6927 \\
Trained MC-dropout & $q_{05}$ & 2901 & 0.6834 & 0.6972 & 2901 & 0.6832 & 0.6973 \\
Trained MC-dropout & $q_{25}$ & 3000 & 0.6756 & 0.6984 & 3000 & 0.6756 & 0.6985 \\
Trained MC-dropout & $q_{50}$ & 3000 & 0.6690 & 0.6975 & 3000 & 0.6690 & 0.6976 \\
Trained MC-dropout & $q_{75}$ & 3000 & 0.6612 & 0.6930 & 3000 & 0.6612 & 0.6930 \\
Trained MC-dropout & $q_{95}$ & 3000 & 0.6484 & 0.6852 & 3000 & 0.6484 & 0.6852 \\
\end{longtable}
\endgroup
\section{Selected Paired Comparisons}\label{app:paired}
Table~\ref{tb:paired} reports the lower-tail comparison for all inference methods, complementing the complete summaries above. Intervals are unadjusted and significance is determined by Holm-adjusted $p$-values.
\begingroup\scriptsize
\setlength{\tabcolsep}{3pt}
\renewcommand{\arraystretch}{1.06}
\begin{longtable}{lllrrrrr}
\caption{$q_{05}$ paired differences from deployed MAP, except trained MC-dropout from its own dropout-disabled model. $d_z$ is the paired standardized effect. LA uses the primary selection rule unless marked EB (empirical Bayes).}\label{tb:paired}\\
\toprule
Target & Method & Metric & $n$ & $\Delta$ & 95\% CI & $d_z$ & $p_{\mathrm{Holm}}$ \\
\midrule\endfirsthead
\toprule
Target & Method & Metric & $n$ & $\Delta$ & 95\% CI & $d_z$ & $p_{\mathrm{Holm}}$ \\
\midrule\endhead
\midrule\multicolumn{8}{r}{Continued on next page}\\\endfoot
\bottomrule\endlastfoot
$y$ & All/diag & RMA & 2999 & +0.0148 & [0.0110, 0.0187] & 0.138 & <0.001 \\
$y$ & All/diag & IoU & 2999 & -0.0341 & [-0.0389, -0.0294] & -0.259 & <0.001 \\
$y$ & LL/diag & RMA & 3000 & +0.0084 & [0.0047, 0.0121] & 0.082 & <0.001 \\
$y$ & LL/diag & IoU & 3000 & -0.0172 & [-0.0211, -0.0134] & -0.162 & <0.001 \\
$y$ & LL/KFAC & RMA & 3000 & +0.0089 & [0.0054, 0.0124] & 0.091 & <0.001 \\
$y$ & LL/KFAC & IoU & 3000 & -0.0227 & [-0.0267, -0.0187] & -0.203 & <0.001 \\
$y$ & LL/full & RMA & 3000 & +0.0076 & [0.0041, 0.0111] & 0.077 & <0.001 \\
$y$ & LL/full & IoU & 3000 & -0.0213 & [-0.0254, -0.0172] & -0.185 & <0.001 \\
$y$ & All/diag (EB) & RMA & 2957 & +0.0075 & [0.0032, 0.0118] & 0.063 & 0.003 \\
$y$ & All/diag (EB) & IoU & 2957 & -0.0512 & [-0.0572, -0.0451] & -0.305 & <0.001 \\
$y$ & LL/diag (EB) & RMA & 2999 & +0.0081 & [0.0045, 0.0117] & 0.080 & <0.001 \\
$y$ & LL/diag (EB) & IoU & 2999 & -0.0182 & [-0.0219, -0.0144] & -0.174 & <0.001 \\
$y$ & LL/KFAC (EB) & RMA & 3000 & +0.0096 & [0.0064, 0.0129] & 0.106 & <0.001 \\
$y$ & LL/KFAC (EB) & IoU & 3000 & -0.0122 & [-0.0156, -0.0088] & -0.128 & <0.001 \\
$y$ & LL/full (EB) & RMA & 2972 & -0.0180 & [-0.0218, -0.0141] & -0.170 & <0.001 \\
$y$ & LL/full (EB) & IoU & 2972 & -0.0641 & [-0.0705, -0.0578] & -0.365 & <0.001 \\
$y$ & Ensemble-5 & RMA & 3000 & +0.0264 & [0.0216, 0.0312] & 0.198 & <0.001 \\
$y$ & Ensemble-5 & IoU & 3000 & -0.0066 & [-0.0126, -0.0006] & -0.039 & 0.188 \\
$y$ & Post-hoc dropout & RMA & 3000 & +0.0152 & [0.0120, 0.0185] & 0.167 & <0.001 \\
$y$ & Post-hoc dropout & IoU & 3000 & -0.0127 & [-0.0162, -0.0092] & -0.129 & <0.001 \\
$y$ & Trained MC-dropout & RMA & 2901 & +0.0135 & [0.0115, 0.0155] & 0.246 & <0.001 \\
$y$ & Trained MC-dropout & IoU & 2901 & +0.0037 & [0.0001, 0.0073] & 0.037 & 0.139 \\
$\hat y$ & All/diag & RMA & 2999 & +0.0148 & [0.0109, 0.0187] & 0.137 & <0.001 \\
$\hat y$ & All/diag & IoU & 2999 & -0.0338 & [-0.0385, -0.0291] & -0.259 & <0.001 \\
$\hat y$ & LL/diag & RMA & 3000 & +0.0084 & [0.0047, 0.0121] & 0.082 & <0.001 \\
$\hat y$ & LL/diag & IoU & 3000 & -0.0169 & [-0.0206, -0.0131] & -0.161 & <0.001 \\
$\hat y$ & LL/KFAC & RMA & 3000 & +0.0089 & [0.0054, 0.0124] & 0.090 & <0.001 \\
$\hat y$ & LL/KFAC & IoU & 3000 & -0.0224 & [-0.0264, -0.0185] & -0.204 & <0.001 \\
$\hat y$ & LL/full & RMA & 3000 & +0.0075 & [0.0040, 0.0110] & 0.077 & <0.001 \\
$\hat y$ & LL/full & IoU & 3000 & -0.0210 & [-0.0251, -0.0169] & -0.184 & <0.001 \\
$\hat y$ & All/diag (EB) & RMA & 2957 & +0.0074 & [0.0031, 0.0117] & 0.062 & 0.004 \\
$\hat y$ & All/diag (EB) & IoU & 2957 & -0.0508 & [-0.0569, -0.0448] & -0.305 & <0.001 \\
$\hat y$ & LL/diag (EB) & RMA & 2999 & +0.0081 & [0.0045, 0.0117] & 0.080 & <0.001 \\
$\hat y$ & LL/diag (EB) & IoU & 2999 & -0.0177 & [-0.0213, -0.0140] & -0.172 & <0.001 \\
$\hat y$ & LL/KFAC (EB) & RMA & 3000 & +0.0097 & [0.0064, 0.0129] & 0.107 & <0.001 \\
$\hat y$ & LL/KFAC (EB) & IoU & 3000 & -0.0118 & [-0.0152, -0.0085] & -0.126 & <0.001 \\
$\hat y$ & LL/full (EB) & RMA & 2969 & -0.0180 & [-0.0218, -0.0142] & -0.170 & <0.001 \\
$\hat y$ & LL/full (EB) & IoU & 2969 & -0.0639 & [-0.0701, -0.0576] & -0.365 & <0.001 \\
$\hat y$ & Ensemble-5 & RMA & 3000 & +0.0264 & [0.0216, 0.0312] & 0.198 & <0.001 \\
$\hat y$ & Ensemble-5 & IoU & 3000 & -0.0062 & [-0.0122, -0.0003] & -0.037 & 0.245 \\
$\hat y$ & Post-hoc dropout & RMA & 3000 & +0.0152 & [0.0120, 0.0185] & 0.167 & <0.001 \\
$\hat y$ & Post-hoc dropout & IoU & 3000 & -0.0125 & [-0.0160, -0.0091] & -0.131 & <0.001 \\
$\hat y$ & Trained MC-dropout & RMA & 2901 & +0.0133 & [0.0113, 0.0152] & 0.248 & <0.001 \\
$\hat y$ & Trained MC-dropout & IoU & 2901 & +0.0038 & [0.0002, 0.0074] & 0.038 & 0.121 \\
\end{longtable}
\endgroup
\section{Complete Budget Sensitivity}\label{app:budgets}
Table~\ref{tb:budgets} reports all configurations in the supplementary budget study. These are valid-instance means on the 150-signal subset rather than class-macro means after exclusions. Ensemble sampling is fixed at five members. Excluding zero summaries can change the effective class composition.
\begingroup\scriptsize
\setlength{\tabcolsep}{3pt}
\renewcommand{\arraystretch}{1.06}
\begin{longtable}{lrlrrrrrr}
\caption{Full supplementary localization-budget results. The valid count $n$ is out of 150. LIME means are conditional on its 64 valid signals.}\label{tb:budgets}\\
\toprule
Method & $S$ & Summary & $n_y$ & RMA $y$ & IoU $y$ & $n_{\hat y}$ & RMA $\hat y$ & IoU $\hat y$ \\
\midrule\endfirsthead
\toprule
Method & $S$ & Summary & $n_y$ & RMA $y$ & IoU $y$ & $n_{\hat y}$ & RMA $\hat y$ & IoU $\hat y$ \\
\midrule\endhead
\midrule\multicolumn{9}{r}{Continued on next page}\\\endfoot
\bottomrule\endlastfoot
All/diag & 5 & Mean & 150 & 0.615 & 0.646 & 150 & 0.615 & 0.646 \\
All/diag & 5 & $q_{05}$ & 148 & 0.642 & 0.625 & 148 & 0.642 & 0.625 \\
All/diag & 5 & $q_{50}$ & 150 & 0.627 & 0.652 & 150 & 0.627 & 0.652 \\
All/diag & 5 & $q_{95}$ & 150 & 0.607 & 0.634 & 150 & 0.607 & 0.634 \\
All/diag & 20 & Mean & 150 & 0.614 & 0.647 & 150 & 0.614 & 0.647 \\
All/diag & 20 & $q_{05}$ & 142 & 0.644 & 0.616 & 142 & 0.644 & 0.616 \\
All/diag & 20 & $q_{50}$ & 150 & 0.625 & 0.656 & 150 & 0.625 & 0.656 \\
All/diag & 20 & $q_{95}$ & 150 & 0.608 & 0.637 & 150 & 0.608 & 0.637 \\
All/diag & 50 & Mean & 150 & 0.613 & 0.655 & 150 & 0.613 & 0.655 \\
All/diag & 50 & $q_{05}$ & 149 & 0.654 & 0.640 & 149 & 0.654 & 0.640 \\
All/diag & 50 & $q_{50}$ & 150 & 0.623 & 0.658 & 150 & 0.623 & 0.658 \\
All/diag & 50 & $q_{95}$ & 150 & 0.608 & 0.643 & 150 & 0.608 & 0.643 \\
All/diag & 100 & Mean & 150 & 0.614 & 0.659 & 150 & 0.614 & 0.659 \\
All/diag & 100 & $q_{05}$ & 149 & 0.658 & 0.645 & 149 & 0.658 & 0.645 \\
All/diag & 100 & $q_{50}$ & 150 & 0.623 & 0.658 & 150 & 0.623 & 0.658 \\
All/diag & 100 & $q_{95}$ & 150 & 0.609 & 0.651 & 150 & 0.609 & 0.651 \\
LL/diag & 5 & Mean & 150 & 0.609 & 0.651 & 150 & 0.609 & 0.651 \\
LL/diag & 5 & $q_{05}$ & 150 & 0.647 & 0.655 & 150 & 0.647 & 0.655 \\
LL/diag & 5 & $q_{50}$ & 150 & 0.611 & 0.647 & 150 & 0.611 & 0.647 \\
LL/diag & 5 & $q_{95}$ & 150 & 0.605 & 0.643 & 150 & 0.605 & 0.643 \\
LL/diag & 20 & Mean & 150 & 0.609 & 0.657 & 150 & 0.609 & 0.657 \\
LL/diag & 20 & $q_{05}$ & 147 & 0.641 & 0.640 & 147 & 0.641 & 0.640 \\
LL/diag & 20 & $q_{50}$ & 150 & 0.613 & 0.656 & 150 & 0.613 & 0.656 \\
LL/diag & 20 & $q_{95}$ & 150 & 0.607 & 0.648 & 150 & 0.607 & 0.648 \\
LL/diag & 50 & Mean & 150 & 0.609 & 0.654 & 150 & 0.609 & 0.654 \\
LL/diag & 50 & $q_{05}$ & 150 & 0.651 & 0.656 & 150 & 0.651 & 0.656 \\
LL/diag & 50 & $q_{50}$ & 150 & 0.614 & 0.657 & 150 & 0.614 & 0.657 \\
LL/diag & 50 & $q_{95}$ & 150 & 0.606 & 0.648 & 150 & 0.606 & 0.648 \\
LL/diag & 100 & Mean & 150 & 0.610 & 0.659 & 150 & 0.610 & 0.659 \\
LL/diag & 100 & $q_{05}$ & 150 & 0.651 & 0.652 & 150 & 0.651 & 0.652 \\
LL/diag & 100 & $q_{50}$ & 150 & 0.615 & 0.659 & 150 & 0.615 & 0.659 \\
LL/diag & 100 & $q_{95}$ & 150 & 0.606 & 0.655 & 150 & 0.606 & 0.655 \\
LL/KFAC & 5 & Mean & 150 & 0.610 & 0.651 & 150 & 0.610 & 0.651 \\
LL/KFAC & 5 & $q_{05}$ & 150 & 0.636 & 0.645 & 150 & 0.636 & 0.645 \\
LL/KFAC & 5 & $q_{50}$ & 150 & 0.615 & 0.651 & 150 & 0.615 & 0.651 \\
LL/KFAC & 5 & $q_{95}$ & 150 & 0.605 & 0.644 & 150 & 0.605 & 0.644 \\
LL/KFAC & 20 & Mean & 150 & 0.606 & 0.649 & 150 & 0.606 & 0.649 \\
LL/KFAC & 20 & $q_{05}$ & 150 & 0.657 & 0.648 & 150 & 0.657 & 0.648 \\
LL/KFAC & 20 & $q_{50}$ & 150 & 0.615 & 0.658 & 150 & 0.616 & 0.658 \\
LL/KFAC & 20 & $q_{95}$ & 150 & 0.600 & 0.644 & 150 & 0.600 & 0.644 \\
LL/KFAC & 50 & Mean & 150 & 0.608 & 0.657 & 150 & 0.608 & 0.657 \\
LL/KFAC & 50 & $q_{05}$ & 150 & 0.654 & 0.666 & 150 & 0.654 & 0.666 \\
LL/KFAC & 50 & $q_{50}$ & 150 & 0.613 & 0.656 & 150 & 0.613 & 0.656 \\
LL/KFAC & 50 & $q_{95}$ & 150 & 0.606 & 0.653 & 150 & 0.606 & 0.653 \\
LL/KFAC & 100 & Mean & 150 & 0.609 & 0.653 & 150 & 0.609 & 0.653 \\
LL/KFAC & 100 & $q_{05}$ & 150 & 0.656 & 0.657 & 150 & 0.656 & 0.657 \\
LL/KFAC & 100 & $q_{50}$ & 150 & 0.614 & 0.656 & 150 & 0.614 & 0.656 \\
LL/KFAC & 100 & $q_{95}$ & 150 & 0.606 & 0.651 & 150 & 0.606 & 0.651 \\
LL/full & 5 & Mean & 150 & 0.614 & 0.654 & 150 & 0.614 & 0.654 \\
LL/full & 5 & $q_{05}$ & 150 & 0.655 & 0.662 & 150 & 0.655 & 0.662 \\
LL/full & 5 & $q_{50}$ & 150 & 0.618 & 0.656 & 150 & 0.618 & 0.656 \\
LL/full & 5 & $q_{95}$ & 150 & 0.608 & 0.654 & 150 & 0.608 & 0.654 \\
LL/full & 20 & Mean & 150 & 0.611 & 0.655 & 150 & 0.611 & 0.655 \\
LL/full & 20 & $q_{05}$ & 148 & 0.646 & 0.626 & 148 & 0.646 & 0.626 \\
LL/full & 20 & $q_{50}$ & 150 & 0.618 & 0.658 & 150 & 0.618 & 0.658 \\
LL/full & 20 & $q_{95}$ & 150 & 0.605 & 0.648 & 150 & 0.605 & 0.648 \\
LL/full & 50 & Mean & 150 & 0.611 & 0.661 & 150 & 0.611 & 0.661 \\
LL/full & 50 & $q_{05}$ & 150 & 0.652 & 0.659 & 150 & 0.652 & 0.659 \\
LL/full & 50 & $q_{50}$ & 150 & 0.615 & 0.662 & 150 & 0.615 & 0.662 \\
LL/full & 50 & $q_{95}$ & 150 & 0.607 & 0.662 & 150 & 0.607 & 0.662 \\
LL/full & 100 & Mean & 150 & 0.611 & 0.659 & 150 & 0.611 & 0.659 \\
LL/full & 100 & $q_{05}$ & 150 & 0.653 & 0.660 & 150 & 0.653 & 0.660 \\
LL/full & 100 & $q_{50}$ & 150 & 0.616 & 0.664 & 150 & 0.616 & 0.664 \\
LL/full & 100 & $q_{95}$ & 150 & 0.606 & 0.660 & 150 & 0.606 & 0.660 \\
Ensemble-5 & 5 & Mean & 150 & 0.643 & 0.676 & 150 & 0.643 & 0.676 \\
Ensemble-5 & 5 & $q_{05}$ & 150 & 0.687 & 0.669 & 150 & 0.687 & 0.669 \\
Ensemble-5 & 5 & $q_{50}$ & 150 & 0.650 & 0.674 & 150 & 0.650 & 0.674 \\
Ensemble-5 & 5 & $q_{95}$ & 150 & 0.631 & 0.669 & 150 & 0.631 & 0.669 \\
Post-hoc dropout & 5 & Mean & 150 & 0.620 & 0.666 & 150 & 0.620 & 0.666 \\
Post-hoc dropout & 5 & $q_{05}$ & 150 & 0.648 & 0.654 & 150 & 0.648 & 0.654 \\
Post-hoc dropout & 5 & $q_{50}$ & 150 & 0.628 & 0.666 & 150 & 0.628 & 0.666 \\
Post-hoc dropout & 5 & $q_{95}$ & 150 & 0.613 & 0.662 & 150 & 0.613 & 0.663 \\
Post-hoc dropout & 20 & Mean & 150 & 0.622 & 0.667 & 150 & 0.622 & 0.667 \\
Post-hoc dropout & 20 & $q_{05}$ & 150 & 0.669 & 0.656 & 150 & 0.669 & 0.656 \\
Post-hoc dropout & 20 & $q_{50}$ & 150 & 0.633 & 0.666 & 150 & 0.633 & 0.666 \\
Post-hoc dropout & 20 & $q_{95}$ & 150 & 0.611 & 0.660 & 150 & 0.611 & 0.660 \\
Post-hoc dropout & 50 & Mean & 150 & 0.619 & 0.662 & 150 & 0.619 & 0.662 \\
Post-hoc dropout & 50 & $q_{05}$ & 150 & 0.661 & 0.670 & 150 & 0.661 & 0.670 \\
Post-hoc dropout & 50 & $q_{50}$ & 150 & 0.629 & 0.668 & 150 & 0.629 & 0.668 \\
Post-hoc dropout & 50 & $q_{95}$ & 150 & 0.607 & 0.655 & 150 & 0.607 & 0.655 \\
Post-hoc dropout & 100 & Mean & 150 & 0.620 & 0.663 & 150 & 0.620 & 0.663 \\
Post-hoc dropout & 100 & $q_{05}$ & 150 & 0.662 & 0.665 & 150 & 0.662 & 0.665 \\
Post-hoc dropout & 100 & $q_{50}$ & 150 & 0.630 & 0.671 & 150 & 0.630 & 0.671 \\
Post-hoc dropout & 100 & $q_{95}$ & 150 & 0.609 & 0.656 & 150 & 0.609 & 0.656 \\
Trained MC-dropout & 5 & Mean & 150 & 0.674 & 0.695 & 150 & 0.674 & 0.695 \\
Trained MC-dropout & 5 & $q_{05}$ & 149 & 0.685 & 0.692 & 149 & 0.685 & 0.692 \\
Trained MC-dropout & 5 & $q_{50}$ & 150 & 0.681 & 0.697 & 150 & 0.681 & 0.697 \\
Trained MC-dropout & 5 & $q_{95}$ & 150 & 0.662 & 0.685 & 150 & 0.662 & 0.685 \\
Trained MC-dropout & 20 & Mean & 150 & 0.674 & 0.698 & 150 & 0.674 & 0.698 \\
Trained MC-dropout & 20 & $q_{05}$ & 147 & 0.690 & 0.696 & 147 & 0.690 & 0.696 \\
Trained MC-dropout & 20 & $q_{50}$ & 150 & 0.681 & 0.702 & 150 & 0.681 & 0.702 \\
Trained MC-dropout & 20 & $q_{95}$ & 150 & 0.659 & 0.685 & 150 & 0.659 & 0.685 \\
Trained MC-dropout & 50 & Mean & 150 & 0.673 & 0.694 & 150 & 0.673 & 0.694 \\
Trained MC-dropout & 50 & $q_{05}$ & 148 & 0.690 & 0.701 & 148 & 0.690 & 0.701 \\
Trained MC-dropout & 50 & $q_{50}$ & 150 & 0.680 & 0.699 & 150 & 0.680 & 0.699 \\
Trained MC-dropout & 50 & $q_{95}$ & 150 & 0.654 & 0.680 & 150 & 0.654 & 0.680 \\
Trained MC-dropout & 100 & Mean & 150 & 0.672 & 0.694 & 150 & 0.672 & 0.694 \\
Trained MC-dropout & 100 & $q_{05}$ & 146 & 0.687 & 0.697 & 146 & 0.687 & 0.697 \\
Trained MC-dropout & 100 & $q_{50}$ & 150 & 0.680 & 0.696 & 150 & 0.680 & 0.696 \\
Trained MC-dropout & 100 & $q_{95}$ & 150 & 0.655 & 0.677 & 150 & 0.655 & 0.677 \\
Dropout control & 1 & MAP & 150 & 0.681 & 0.685 & 150 & 0.681 & 0.685 \\
MAP & 1 & MAP & 150 & 0.639 & 0.666 & 150 & 0.639 & 0.666 \\
LIME-128 & 1 & MAP & 64 & 0.816 & 0.741 & 64 & 0.816 & 0.741 \\
KernelSHAP-128 & 1 & MAP & 150 & 0.625 & 0.601 & 150 & 0.625 & 0.601 \\
\end{longtable}
\endgroup
\section{Complete Faithfulness Summaries}\label{app:faith}
Table~\ref{tb:faithall} reports every supplementary faithfulness configuration, including the separate dropout-trained scorer. Table~\ref{tb:faithpaired} provides paired effects for the deployed-classifier comparisons. MC-dropout is paired against its own control.
\begingroup\scriptsize
\setlength{\tabcolsep}{3pt}
\renewcommand{\arraystretch}{1.06}
\begin{longtable}{lllrrrr}
\caption{Faithfulness with predicted-class targets. LA and dropout use $S=100$, ensemble uses five members.}\label{tb:faithall}\\
\toprule
Scorer & Method & Summary & $n$ & Deletion & Insertion & Mono. \\
\midrule\endfirsthead
\toprule
Scorer & Method & Summary & $n$ & Deletion & Insertion & Mono. \\
\midrule\endhead
\midrule\multicolumn{7}{r}{Continued on next page}\\\endfoot
\bottomrule\endlastfoot
Deployed & All/diag & Mean & 150 & 0.322 & 0.484 & 0.208 \\
Deployed & All/diag & $q_{05}$ & 149 & 0.327 & 0.486 & 0.129 \\
Deployed & All/diag & $q_{95}$ & 150 & 0.343 & 0.481 & 0.160 \\
Deployed & LL/diag & Mean & 150 & 0.331 & 0.480 & 0.234 \\
Deployed & LL/diag & $q_{05}$ & 150 & 0.326 & 0.481 & 0.192 \\
Deployed & LL/diag & $q_{95}$ & 150 & 0.357 & 0.477 & 0.184 \\
Deployed & LL/KFAC & Mean & 150 & 0.333 & 0.482 & 0.215 \\
Deployed & LL/KFAC & $q_{05}$ & 150 & 0.328 & 0.481 & 0.201 \\
Deployed & LL/KFAC & $q_{95}$ & 150 & 0.365 & 0.480 & 0.181 \\
Deployed & LL/full & Mean & 150 & 0.330 & 0.479 & 0.216 \\
Deployed & LL/full & $q_{05}$ & 150 & 0.334 & 0.478 & 0.186 \\
Deployed & LL/full & $q_{95}$ & 150 & 0.336 & 0.475 & 0.196 \\
Deployed & Ensemble-5 & Mean & 150 & 0.341 & 0.487 & 0.175 \\
Deployed & Ensemble-5 & $q_{05}$ & 150 & 0.333 & 0.484 & 0.129 \\
Deployed & Ensemble-5 & $q_{95}$ & 150 & 0.348 & 0.479 & 0.158 \\
Deployed & Post-hoc dropout & Mean & 150 & 0.328 & 0.484 & 0.225 \\
Deployed & Post-hoc dropout & $q_{05}$ & 150 & 0.313 & 0.481 & 0.219 \\
Deployed & Post-hoc dropout & $q_{95}$ & 150 & 0.341 & 0.472 & 0.212 \\
Own & Trained MC-dropout & Mean & 150 & 0.310 & 0.467 & 0.169 \\
Own & Trained MC-dropout & $q_{05}$ & 146 & 0.320 & 0.463 & 0.175 \\
Own & Trained MC-dropout & $q_{95}$ & 150 & 0.315 & 0.476 & 0.169 \\
Deployed & MAP & MAP & 150 & 0.309 & 0.487 & 0.224 \\
Own & Dropout control & MAP & 150 & 0.318 & 0.477 & 0.149 \\
Deployed & LIME-128 & MAP & 64 & 0.469 & 0.769 & 0.092 \\
Deployed & KernelSHAP-128 & MAP & 150 & 0.423 & 0.487 & 0.014 \\
\end{longtable}
\endgroup\begingroup\scriptsize
\setlength{\tabcolsep}{3pt}
\renewcommand{\arraystretch}{1.06}
\begin{longtable}{lllrrrrr}
\caption{Faithfulness paired effects on common valid signals. All MC-dropout rows use their own dropout-disabled baseline. The remaining rows use deployed MAP. Negative deletion differences and positive insertion/monotonicity differences are favorable.}\label{tb:faithpaired}\\
\toprule
Method & Summary & Metric & $n$ & $\Delta$ & 95\% CI & $d_z$ & $p_{\mathrm{Holm}}$ \\
\midrule\endfirsthead
\toprule
Method & Summary & Metric & $n$ & $\Delta$ & 95\% CI & $d_z$ & $p_{\mathrm{Holm}}$ \\
\midrule\endhead
\midrule\multicolumn{8}{r}{Continued on next page}\\\endfoot
\bottomrule\endlastfoot
Post-hoc dropout & Mean & Del. & 150 & +0.019 & [0.004, 0.033] & 0.211 & 0.110 \\
Post-hoc dropout & $q_{05}$ & Del. & 150 & +0.004 & [-0.009, 0.016] & 0.046 & 1.000 \\
Post-hoc dropout & $q_{95}$ & Del. & 150 & +0.032 & [0.009, 0.054] & 0.229 & 0.068 \\
Ensemble-5 & Mean & Del. & 150 & +0.032 & [0.007, 0.057] & 0.207 & 0.110 \\
Ensemble-5 & $q_{05}$ & Del. & 150 & +0.024 & [-0.001, 0.048] & 0.158 & 0.221 \\
Ensemble-5 & $q_{95}$ & Del. & 150 & +0.039 & [0.012, 0.067] & 0.232 & 0.067 \\
KernelSHAP-128 & MAP & Del. & 150 & +0.114 & [0.083, 0.146] & 0.587 & <0.001 \\
All/diag & Mean & Del. & 150 & +0.013 & [-0.002, 0.028] & 0.141 & 0.258 \\
All/diag & $q_{05}$ & Del. & 149 & +0.017 & [0.007, 0.026] & 0.277 & 0.015 \\
All/diag & $q_{95}$ & Del. & 150 & +0.034 & [0.011, 0.056] & 0.244 & 0.045 \\
LL/diag & Mean & Del. & 150 & +0.022 & [0.004, 0.040] & 0.201 & 0.110 \\
LL/diag & $q_{05}$ & Del. & 150 & +0.017 & [0.004, 0.029] & 0.213 & 0.110 \\
LL/diag & $q_{95}$ & Del. & 150 & +0.048 & [0.020, 0.076] & 0.279 & 0.014 \\
LL/full & Mean & Del. & 150 & +0.021 & [0.004, 0.039] & 0.201 & 0.110 \\
LL/full & $q_{05}$ & Del. & 150 & +0.025 & [0.010, 0.039] & 0.280 & 0.014 \\
LL/full & $q_{95}$ & Del. & 150 & +0.027 & [0.002, 0.051] & 0.177 & 0.161 \\
LL/KFAC & Mean & Del. & 150 & +0.024 & [0.005, 0.043] & 0.208 & 0.110 \\
LL/KFAC & $q_{05}$ & Del. & 150 & +0.019 & [0.007, 0.030] & 0.268 & 0.019 \\
LL/KFAC & $q_{95}$ & Del. & 150 & +0.056 & [0.030, 0.082] & 0.352 & <0.001 \\
LIME-128 & MAP & Del. & 64 & -0.006 & [-0.056, 0.045] & -0.028 & 1.000 \\
Trained MC-dropout & Mean & Del. & 150 & -0.008 & [-0.019, 0.004] & -0.108 & 0.566 \\
Trained MC-dropout & $q_{05}$ & Del. & 146 & +0.001 & [-0.009, 0.011] & 0.010 & 1.000 \\
Trained MC-dropout & $q_{95}$ & Del. & 150 & -0.002 & [-0.017, 0.012] & -0.027 & 1.000 \\
Post-hoc dropout & Mean & Ins. & 150 & -0.004 & [-0.012, 0.005] & -0.069 & 1.000 \\
Post-hoc dropout & $q_{05}$ & Ins. & 150 & -0.007 & [-0.018, 0.004] & -0.102 & 1.000 \\
Post-hoc dropout & $q_{95}$ & Ins. & 150 & -0.015 & [-0.030, -0.001] & -0.171 & 0.751 \\
Ensemble-5 & Mean & Ins. & 150 & -0.000 & [-0.010, 0.009] & -0.008 & 1.000 \\
Ensemble-5 & $q_{05}$ & Ins. & 150 & -0.003 & [-0.020, 0.014] & -0.030 & 1.000 \\
Ensemble-5 & $q_{95}$ & Ins. & 150 & -0.008 & [-0.022, 0.006] & -0.091 & 1.000 \\
KernelSHAP-128 & MAP & Ins. & 150 & -0.000 & [-0.030, 0.030] & -0.001 & 1.000 \\
All/diag & Mean & Ins. & 150 & -0.003 & [-0.014, 0.008] & -0.044 & 1.000 \\
All/diag & $q_{05}$ & Ins. & 149 & -0.005 & [-0.019, 0.010] & -0.053 & 1.000 \\
All/diag & $q_{95}$ & Ins. & 150 & -0.006 & [-0.023, 0.011] & -0.060 & 1.000 \\
LL/diag & Mean & Ins. & 150 & -0.007 & [-0.021, 0.007] & -0.084 & 1.000 \\
LL/diag & $q_{05}$ & Ins. & 150 & -0.006 & [-0.018, 0.006] & -0.082 & 1.000 \\
LL/diag & $q_{95}$ & Ins. & 150 & -0.010 & [-0.025, 0.005] & -0.108 & 1.000 \\
LL/full & Mean & Ins. & 150 & -0.009 & [-0.022, 0.005] & -0.101 & 1.000 \\
LL/full & $q_{05}$ & Ins. & 150 & -0.010 & [-0.026, 0.007] & -0.097 & 1.000 \\
LL/full & $q_{95}$ & Ins. & 150 & -0.013 & [-0.028, 0.003] & -0.134 & 1.000 \\
LL/KFAC & Mean & Ins. & 150 & -0.005 & [-0.018, 0.008] & -0.062 & 1.000 \\
LL/KFAC & $q_{05}$ & Ins. & 150 & -0.006 & [-0.020, 0.007] & -0.079 & 1.000 \\
LL/KFAC & $q_{95}$ & Ins. & 150 & -0.007 & [-0.023, 0.008] & -0.077 & 1.000 \\
LIME-128 & MAP & Ins. & 64 & +0.043 & [0.002, 0.083] & 0.265 & 0.751 \\
Trained MC-dropout & Mean & Ins. & 150 & -0.010 & [-0.018, -0.002] & -0.193 & 0.058 \\
Trained MC-dropout & $q_{05}$ & Ins. & 146 & -0.020 & [-0.036, -0.003] & -0.196 & 0.058 \\
Trained MC-dropout & $q_{95}$ & Ins. & 150 & -0.001 & [-0.009, 0.008] & -0.012 & 0.881 \\
Post-hoc dropout & Mean & Mono. & 150 & +0.001 & [-0.017, 0.019] & 0.008 & 1.000 \\
Post-hoc dropout & $q_{05}$ & Mono. & 150 & -0.005 & [-0.023, 0.014] & -0.041 & 1.000 \\
Post-hoc dropout & $q_{95}$ & Mono. & 150 & -0.011 & [-0.031, 0.008] & -0.095 & 1.000 \\
Ensemble-5 & Mean & Mono. & 150 & -0.049 & [-0.070, -0.027] & -0.363 & <0.001 \\
Ensemble-5 & $q_{05}$ & Mono. & 150 & -0.095 & [-0.125, -0.066] & -0.522 & <0.001 \\
Ensemble-5 & $q_{95}$ & Mono. & 150 & -0.065 & [-0.093, -0.038] & -0.379 & <0.001 \\
KernelSHAP-128 & MAP & Mono. & 150 & -0.210 & [-0.248, -0.172] & -0.899 & <0.001 \\
All/diag & Mean & Mono. & 150 & -0.016 & [-0.044, 0.012] & -0.090 & 1.000 \\
All/diag & $q_{05}$ & Mono. & 149 & -0.098 & [-0.129, -0.067] & -0.518 & <0.001 \\
All/diag & $q_{95}$ & Mono. & 150 & -0.064 & [-0.091, -0.038] & -0.391 & <0.001 \\
LL/diag & Mean & Mono. & 150 & +0.010 & [-0.013, 0.032] & 0.071 & 1.000 \\
LL/diag & $q_{05}$ & Mono. & 150 & -0.032 & [-0.060, -0.004] & -0.182 & 0.247 \\
LL/diag & $q_{95}$ & Mono. & 150 & -0.039 & [-0.061, -0.018] & -0.293 & 0.005 \\
LL/full & Mean & Mono. & 150 & -0.008 & [-0.028, 0.013] & -0.060 & 1.000 \\
LL/full & $q_{05}$ & Mono. & 150 & -0.037 & [-0.063, -0.012] & -0.238 & 0.045 \\
LL/full & $q_{95}$ & Mono. & 150 & -0.028 & [-0.049, -0.007] & -0.212 & 0.104 \\
LL/KFAC & Mean & Mono. & 150 & -0.009 & [-0.029, 0.010] & -0.075 & 1.000 \\
LL/KFAC & $q_{05}$ & Mono. & 150 & -0.023 & [-0.047, 0.001] & -0.157 & 0.457 \\
LL/KFAC & $q_{95}$ & Mono. & 150 & -0.042 & [-0.065, -0.020] & -0.302 & 0.004 \\
LIME-128 & MAP & Mono. & 64 & -0.182 & [-0.241, -0.123] & -0.773 & <0.001 \\
Trained MC-dropout & Mean & Mono. & 150 & +0.019 & [-0.001, 0.040] & 0.150 & 0.180 \\
Trained MC-dropout & $q_{05}$ & Mono. & 146 & +0.016 & [-0.007, 0.039] & 0.112 & 0.180 \\
Trained MC-dropout & $q_{95}$ & Mono. & 150 & +0.020 & [-0.001, 0.041] & 0.155 & 0.180 \\
\end{longtable}
\endgroup

\bibliographystyle{elsarticle-num}
\bibliography{refer.bib}
\end{document}